\documentclass[journal,twocolumn]{IEEEtran}

\ifx\pdfoutput\undefined
\usepackage{graphicx}
\graphicspath{{eps_figs/}}
\else
\usepackage[pdftex]{graphicx}
\graphicspath{{pdf_figs/}}
\fi

\usepackage{latexsym,amssymb,amsmath,ifthen,afterpage,cite,amsfonts}
\usepackage{mathtools}
\usepackage{lipsum}
\usepackage{subcaption}
\usepackage{algorithm}
\usepackage{algorithmic}
\usepackage{enumerate}
\usepackage{hyperref}
\usepackage{tikz} 
\usepackage{pgfplots}

\pgfplotsset{compat=1.17}

\definecolor{primalblue}{HTML}{1F77B4}
\definecolor{dualred}{HTML}{D62728}

\pgfplotsset{
  gibbsaxis/.style={
    width=\linewidth, height=0.62\linewidth,
    tick label style={font=\small}, label style={font=\small},
    legend style={font=\footnotesize, draw=none, fill opacity=0.85,
                  text opacity=1, cells={anchor=west}},
    grid=both, grid style={dotted, gray!40},
    every axis plot/.append style={line width=0.9pt},
  }
}
\usepackage{pgfplotstable}
\usepackage{filecontents}
\usepgfplotslibrary{statistics}
\usepgfplotslibrary{fillbetween}
\usepgfplotslibrary{groupplots}
\usepackage{booktabs}

\newcommand{\FT}{\mathcal{F}}

\newcommand{\cond}{\hspace{0.02em}|\hspace{0.08em}}

\newcommand{\R}{\mathbb{R}}

\newcommand{\Q}{{\mathbf{Q}}}
\newcommand{\A}{{\mathbf{A}}}
\newcommand{\B}{{\mathbf{B}}}
\newcommand{\Lap}{{\mathbf L}}
\newcommand{\I}{{\mathbf I}}

\newcommand{\E}{\mathcal{E}}
\newcommand{\V}{\mathcal{V}}

\newcommand{\X}{{\bf X}}
\newcommand{\x}{{\bf x}}
\newcommand{\Y}{{\bf Y}}
\newcommand{\y}{{\bf y}}
\newcommand{\z}{{\bf z}}

\newcounter{examplecntr}
\newenvironment{example}[1][]%
{\begin{trivlist}\small\item[]\refstepcounter{examplecntr}%
 {\bfseries Example~\theexamplecntr%
  \ifthenelse{\equal{#1}{}}{}{ (#1)}.
}}%
{\end{trivlist}}

\newcounter{propositioncntr}
\newenvironment{proposition}[1][]%
{\begin{trivlist}\item[]\refstepcounter{propositioncntr}%
{\bfseries Proposition~\thepropositioncntr%
  \ifthenelse{\equal{#1}{}}{}{ (#1)}.
}}%
{\hfill$\Box$\end{trivlist}}

\newcounter{theoremcntr}
{\begin{trivlist}\item[]\refstepcounter{theoremcntr}%
{\bfseries Theorem~\thetheoremcntr%
  \ifthenelse{\equal{#1}{}}{}{ (#1)}.
}}%
{\hfill$\Box$\end{trivlist}}

\newcommand{\pos}[2]{\makebox(0,0)[#1]{#2}}

\begin{document}
\DeclareGraphicsExtensions{.pdf}

\title{Accelerated Random-Sweep Gibbs Sampling for Gaussian Graphical Models via \\ Dual Normal Factor Graphs}

\author{Borna Khodabandeh and Mehdi Molkaraie
\thanks{B.~Khodabandeh is with the Department of Statistics, University of Oxford, Oxford OX1 3LB, UK. M.~Molkaraie is with the Department of Signal Theory and Communications, UPC, 
08034 Barcelona, Spain and with the 
Department of Statistical Sciences, University of Toronto, ON M5G 1X6, Canada (email: mehdi.molkaraie@alumni.epfl.ch). 
}
}

\maketitle 
 
\begin{abstract}
We study the convergence properties of the random-sweep Gibbs sampler 
for Gaussian graphical models with a thin-membrane prior. 
We demonstrate that the convergence rate of the Gibbs sampler is significantly accelerated in the dual model, which is obtained by applying the Fourier transform to 
the local factors of the normal factor graph representing the original model. In both domains, we derive the exact convergence rates for homogeneous $k$-regular graphs. 
We prove that, for all homogeneous models whose graphical representations contain cycles, the convergence rate in the dual 
domain is universal and independent of the underlying graph topology. Moreover, we show that the effective convergence rate in the dual domain is governed by the algebraic 
connectivity of the graph, providing an additional acceleration without increasing the computational complexity per sweep.
We further establish an explicit algebraic relation between the covariance structures of the primal and dual models, enabling marginal statistics of the primal model 
to be recovered directly from those of the dual model. Finally, numerical experiments on several graph families confirm our theoretical results and demonstrate substantial 
improvements in the convergence rates in various settings.

\end{abstract} 



\section{Introduction}
\label{sec:intro}

In this paper, we focus on sampling-based estimation and inference of statistical quantities for a Gaussian graphical model, defined on a finite, simple, and connected graph 
$G=(\V,\E)$ using Gibbs sampling. 

We generally assume that $G$ contains cycles (e.g., $G$ is a $k$-regular or a complete graph) which means $|\E| > |\V| - 1$. 
In fact, if a multivariate distribution has a tree structure, it is
possible to draw independent samples from the distribution and to compute the exact marginal densities efficiently via the sum-product (the belief propagation) algorithm \cite[Chapter 20]{murphy:2012}.  
However, for completeness, we will consider the star graph as a cycle-free example.

As a model, we focus on Gaussian graphical models with a thin-membrane prior defined in \eqref{eqn:PDFP}. The model is important in surface reconstruction 
(as a measure of smoothness \cite{winkler2012},  
\cite{terzopoulos2009regularization,terzopoulos2002computation}) and in vision problems \cite{szeliski1990bayesian, papandreou2010gaussian}. 
In our setting, the structure of the graph and the parameters of the model are assumed to be known. We aim to estimate statistical quantities, such as marginal 
variances, in large graphs. However, the exact computation of these quantities requires inverting the precision matrix with 
computational complexity $\mathcal{O}(|\V|^3)$, which is infeasible for large graphs. 
The Gibbs sampling algorithm exploits the sparsity of the precision matrix to estimate the desired quantities with complexity $\mathcal{O}(|\E|)$ per sweep.

We represent the model using normal factor graphs \cite{Forney:01}. The exact rate of convergence of the random-sweep Gibbs sampler is then derived for a 
class of graphs including
 $k$-regular and balanced complete bipartite graphs.

Next, we consider model transformations based on the
Fourier transform of the local factors of the normal factor graph. In the context of codes on graphs (e.g., LDPC codes) this transformed representation is referred to as the dual
normal factor graph~\cite{Forney:11}.

Our analytical results show that the dual model yields significantly accelerated convergence rates while maintaining virtually the same computational complexity, albeit
at the expense of increased memory usage. Notably, the rates in the dual domain hold for any homogeneous model with cycles; consequently, we establish
the universality of these rates across a 
broad class of graph topologies within our framework. The convergence rate in the dual domain can be further improved, where the gain is governed by the algebraic 
connectivity of $G$. 
Moreover, a direct algebraic relation between the covariance matrices of the primal and the dual models enables the direct computation of the primal marginal variances and the covariance structure 
from samples drawn in the dual domain (see Section~\ref{sec:DualModel}).  

In this paper, we focus exclusively on the baseline random-sweep Gibbs sampling algorithm. 
In both domains, faster convergence rates can, in theory, be achieved by employing blocked or collapsed Gibbs sampling~\cite[Chapter 6]{liu1995covariance}.
However, as these variations of the basic algorithm are applicable to both domains, the advantages of the dual domain carry over to blocked 
or collapsed Gibbs sampling.

We use bold uppercase letters to denote random vectors (e.g., $\X$) and bold lowercase letters for their 
realizations (e.g., $\x$). Matrices are also denoted by bold uppercase letters (e.g., $\A$), where the distinction between matrices and random vectors is clear from the context.

The paper is organized as follows. Section~\ref{sec:Model} introduces the model and derives its precision matrix in terms of the Laplacian of the graph.
The normal realization of the model is discussed in Section~\ref{sec:GM}.
The convergence rate of the random-sweep Gibbs sampler for 
$k$-regular, balanced complete bipartite, and star graphs is 
analyzed in Section~\ref{sec:GibbsP}. Section~\ref{sec:DualModel} introduces the corresponding dual model. We derive the exact rate of convergence of the Gibbs sampler in 
the dual domain in~Section~\ref{sec:GibbsDual}. 
Numerical experiments are presented in Section~\ref{sec:Num}, followed by directions for future research in Section~\ref{sec:Future}.

\section{The Model}
\label{sec:Model}

On $G=(\V,\E)$, we consider a homogeneous multivariate Gaussian distribution with 
the following PDF
\begin{equation}
\label{eqn:PDFP}
f_\X(\x) \propto \prod_{v \in \V}\!\textrm{exp}\Big(-\frac{x_v^2}{2s^2}\Big)\!\prod_{(u,v) \in \E}\!\!\!\textrm{exp}\Big(-\frac{(x_u - x_v)^2}{2\sigma^2}\Big)
\end{equation}
where $\x \in \R^{|\V|}$, and the positive real numbers $s^2$ and $\sigma^2$ are the variances associated with the vertices and the edges of $G$, respectively. 
We denote the precision matrix (i.e., the inverse of the covariance matrix) associated with the PDF in~(\ref{eqn:PDFP})
by the symmetric positive definite matrix $\Q \in \R^{|\V|\times |\V|}$. 
To denote the covariance matrix, we will use $\boldsymbol{\Sigma}$, $\text{Cov}(\X)$, and $\Q^{-1}$ interchangeably.

Next, we give an arbitrary orientation to $G$ by assigning a direction to each edge. E.g., we let $y_e = x_u - x_v$, where 
the edge $e$ is oriented from its head vertex $u$ to its tail vertex $v$. The oriented incidence matrix $\B$ of $G$ with respect to some given 
orientation is $|\V|\times |\E|$, such that $\B[u,e] = +1$, $\B[v,e] = -1$, and 
is 0 otherwise. The rank of $\B$ is $|\V| -1$. For more details, see~\cite[Chapter 9]{Stan2013}.

We denote the variables on the edges by $\Y$, where $\y \in \R^{|\E|}$. In this setting
\begin{equation}
\label{eqn:YfromX}
\Y = \mathbf{B}^\intercal\X
\end{equation} 
i.e., $\Y$ can be computed as a function of $\X$. We thus call $\Y$ auxiliary (or dependent) random variables~\cite{Forney:18}.

Accordingly, we define the following local factors
\begin{IEEEeqnarray}{rCl}
\phi(x_v) & =&  \textrm{exp}\Big(-\frac{x_v^2}{2s^2}\Big),  \quad v \in \V \label{eqn:PSIV1} \\[1mm]
\psi(y_e) & = & \textrm{exp}\Big(-\frac{y_e^2}{2\sigma^2}\Big), \quad e \in \E \label{eqn:PSIV2}
\end{IEEEeqnarray}

Therefore, the PDF in (\ref{eqn:PDFP}) can be expressed as
\begin{equation}
\label{eqn:PDFPr1} 
f_\X(\x) = \frac{1}{Z_f}\prod_{v \in \V}\phi(x_v)\prod_{e\in \E}\psi(y_e) 
\end{equation}
where $Z_f $ is the normalization constant. The marginal density over $X_1$ can be computed as
\begin{equation}
\label{eqn:MargP} 
f_{X_1}(x_1) = \int_{x_2}\int_{x_3}\ldots\int_{x_{|\V|}} f_\X(\x)dx_2dx_3\ldots dx_{|\V|}
\end{equation}

From (\ref{eqn:YfromX}) and (\ref{eqn:PDFPr1}), we obtain
\begin{IEEEeqnarray}{rCl}
f_\X(\x)  & \propto & \textrm{exp}\big(-\frac{1}{2s^2}\x^\intercal\x\big)\textrm{exp}\big(-\frac{1}{2\sigma^2}(\mathbf{B}^\intercal\x)^\intercal\mathbf{B}^\intercal\x\big) \notag \\[1mm]
              & = & \textrm{exp}\big(-\frac{1}{2}\x^\intercal(\frac{1}{s^2}\I + \frac{1}{\sigma^2}\B\B^\intercal)\x\big) \notag \\[1mm]
              & = & \textrm{exp}\big(-\frac{1}{2}\x^\intercal\Q\,\x\big) \label{eqn:PDFPr3}
\end{IEEEeqnarray}
We conclude that the precision matrix is 
\begin{IEEEeqnarray}{rCl}
\Q & = & \frac{1}{s^2}\I + \frac{1}{\sigma^2}\B\B^\intercal \label{eqn:QP1}\\[1mm]
     & = & \frac{1}{s^2}\I + \frac{1}{\sigma^2}\Lap \label{eqn:QP2}
\end{IEEEeqnarray}
where $\Lap = \B\B^\intercal$ is a $|\V|\times |\V|$ symmetric non-negative definite matrix known as the Laplacian of $G$. Note that in order to construct $\B$, we assigned 
arbitrary directions to the edges of $G$. However, all such assignments will give rise to the same
Laplacian matrix~\cite{Stan2013}.

The derivation can be easily extended to non-homogeneous models. Let $\{s_v^2\}_{v=1}^{|\V|}$ and $\{\sigma_e^2\}_{e=1}^{|\E|}$ denote the variances on the vertices and
 on the edges, respectively. In this case
\begin{equation}
\label{eqn:QPNH}
\Q = \mathbf{D}_{s}^{-1} + \B\mathbf{D}_{\sigma}^{-1}\B^\intercal
\end{equation}
where diagonal matrices $\mathbf{D}_{s} = \text{diag$(s_1^2, s_2^2, \ldots, s_{|\V|}^2)$}$ and $\mathbf{D}_{\sigma} = \text{diag$(\sigma_1^2, \sigma_2^2, \ldots, \sigma_{|\E|}^2)$}$.
In non-homogeneous models, the local factors are denoted by $\phi_v(\cdot)$ and $\psi_e(\cdot)$.


\section{Graphical Model of $f_\X(\cdot)$}
\label{sec:GM}

We can represent the factorization of a multivariate function, as in (\ref{eqn:PDFP}) and (\ref{eqn:PDFPr1}), by graphical models. 
In general, the structure of a Gaussian graphical model depends on its precision matrix. Indeed, nonzero entries of $\Q$ indicate the presence of factors in the model 
and off-diagonal zero entries of $\Q$ indicate the 
absence of interactions between
the corresponding variables. This is generally known as the Markov property of Gaussian Markov random fields~\cite{lauritzen1996graphical}. 

Fig.~\ref{fig:2DP1} shows the factor graph of (\ref{eqn:PDFP}) on a 2D lattice, where the filled circles (i.e., the variable nodes) represent 
the variables and the boxes (i.e., the factor nodes) represent the 
factors~\cite{KFL:01}.

In this paper, we consider graphical representations in terms of normal factor graphs, in which, vertices represent
the factors and edges represent the variables. The edge that represents $x_v$  is
connected to the vertex that represents $\phi(\cdot)$ if and only if $x_v$ is an argument of $\phi(\cdot)$.

The description of the normal factor graph seems to impose the restriction that variables should appear in (at most) two 
factors. We can lift this restriction by introducing equality indicator factors to replicate the variables that are 
involved in more than two factors, thereby maintaining the graph's normal structure~\cite{Forney:01, Lg:ifg2004}.

Accordingly, for a subset $\mathcal{S}$ of $\V$, let $\x_{\mathcal{S}} = (x_s, s \in \mathcal{S})$, and define an 
equality indicator factor as
\begin{equation}
\label{eqn:EQ} 
\delta_{=}(\x_{\mathcal{S}}) = 
\begin{cases}
1 & \text{if $x_1 = x_2 = \ldots = x_{|\mathcal{S}|}$}\\[1mm]
0 & \text{otherwise}
\end{cases}
\end{equation}
and a zero-sum indicator factor as
\begin{equation}
\label{eqn:ZSum} 
\delta_{+}(\x_{\mathcal{S}}) = 
\begin{cases}
1 & \text{if $x_1 + x_2 + \ldots + x_{|\mathcal{S}|} = 0$}\\[1mm]
0 & \text{otherwise.}
\end{cases}
\end{equation}

The details of normal factor graphs and the normalization procedure are explained in the following examples.

\begin{example} 
As in Fig.~\ref{fig:2DP1}, we consider a Gaussian Markov random field on a 2D lattice. 
Fig.~\ref{fig:2DP2} shows the normal factor graph of the factorization in~(\ref{eqn:PDFPr1}). The small empty boxes
represent (\ref{eqn:PSIV1}) and the big empty boxes represent (\ref{eqn:PSIV2}).

Equality indicator factors are used to replicate the variables (the edges), e.g., $X_1 = X'_1 = X^{''}_1$.
The symbol $``\circ"$ attached to zero-sum indicator factors is used for sign inversion. E.g., the zero-sum indicator factor in Fig.~\ref{fig:2DP2} imposes
the constraint that $Y_1 + X_1 - X_2 = 0$. It is clear that $\y$ is a function of $\x$.

In a 2D torus (i.e., a 2D lattice with periodic boundary conditions) all equality indicator factors will have degree four, $|\E| = 2|\V|$, and the precision matrix will have the following form
\begin{equation}
Q_{ij}= 
\begin{cases}
\dfrac{1}{s^2} + \dfrac{4}{\sigma^2}& \text{if $i = j$}\\[2mm]
-\dfrac{1}{\sigma^2} & \text{if $(i, j) \in \E$}\\[2mm]
0 & \text{otherwise.}
\end{cases}
\end{equation}
\end{example}

\begin{example}
We look at the factorization (\ref{eqn:PDFPr1} ) on a star graph $S_{|\V|}$. For $|\V| = 5$, The normal factor graph of the model
is illustrated in Fig.~\ref{fig:PStar}, in which the small empty boxes
represent (\ref{eqn:PSIV1}) and the big empty boxes represent (\ref{eqn:PSIV2}). The equality indicator factor in the centre of Fig.~\ref{fig:PStar} enforces the constraint that $X_1 = X'_1 = X^{''}_1 = X^{(3)}_1 = X^{(4)}_1$.

The precision matrix $\Q$ of a homogeneous star graph $S_{|\V|}$ is a symmetric arrowhead matrix\footnote{An arrowhead matrix is a square matrix with zeros in all entries except for 
the main diagonal, the first row, and the first column~\cite{o1990computing, stor2015accurate}.} with its first row equal to

\begin{equation}
\label{eqn:Star1} 
\big(\frac{1}{s^2} + \frac{|\E|}{\sigma^2}, -\frac{1}{\sigma^2}, -\frac{1}{\sigma^2}, \dots, -\frac{1}{\sigma^2}\big)
\end{equation}
and with diagonal entries equal to
\begin{equation}
\label{eqn:Star2} 
\big(\frac{1}{s^2} + \frac{|\E|}{\sigma^2}, \frac{1}{s^2} + \frac{1}{\sigma^2}, \frac{1}{s^2} + \frac{1}{\sigma^2}, \dots, \frac{1}{s^2} + \frac{1}{\sigma^2}\big)
\end{equation}
Here $|\E| = |\V| - 1$ because $S_{|\V|}$ is cycle-free. 
\end{example}

\begin{figure}[t]
\setlength{\unitlength}{0.88mm}
\centering
\begin{picture}(62,72)(-6,-6)
\small
\put(2,62){\circle*{4.0}}
\put(12,60){\framebox(4,4){$$}}
\put(16,62){\line(1,0){8}}
\put(26,62){\circle*{4.0}}
\put(28,62){\line(1,0){8}}         
\put(36,60){\framebox(4,4){$$}}
\put(40,62){\line(1,0){8}}
\put(50,62){\circle*{4.0}}
\put(0,50){\framebox(4,4){$$}}
\put(24,50){\framebox(4,4){$$}}
\put(48,50){\framebox(4,4){$$}}
\put(2,42){\circle*{4.0}}
\put(4,42){\line(1,0){8}}
\put(12,40){\framebox(4,4){$$}}
\put(16,42){\line(1,0){8}}
\put(26,42){\circle*{4.0}}
\put(28,42){\line(1,0){8}}
\put(36,40){\framebox(4,4){$$}}
\put(40,42){\line(1,0){8}}
\put(50,42){\circle*{4.0}}
\put(0,30){\framebox(4,4){$$}}
\put(24,30){\framebox(4,4){$$}}
\put(48,30){\framebox(4,4){$$}}
\put(2,22){\circle*{4.0}}
\put(4,22){\line(1,0){8}}
\put(12,20){\framebox(4,4){$$}}
\put(16,22){\line(1,0){8}}
\put(26,22){\circle*{4.0}}
\put(28,22){\line(1,0){8}}
\put(36,20){\framebox(4,4){$$}}
\put(40,22){\line(1,0){8}}
\put(50,22){\circle*{4.0}}
%
\put(4,62){\line(1,0){8}}        

\put(28,62){\line(1,0){8}}       

\put(4,42){\line(1,0){8}}

\put(28,42){\line(1,0){8}}

\put(4,22){\line(1,0){8}}

\put(28,22){\line(1,0){8}}

\put(2,54){\line(0,1){6}}
\put(2,50){\line(0,-1){6}}
\put(26,54){\line(0,1){6}}
\put(26,50){\line(0,-1){6}}
\put(50,54){\line(0,1){6}}
\put(50,50){\line(0,-1){6}}
\put(2,34){\line(0,1){6}}
\put(2.0,30){\line(0,-1){6}}
\put(26,34){\line(0,1){6}}
\put(26.0,30){\line(0,-1){6}}
\put(50,34){\line(0,1){6}}
\put(50,30){\line(0,-1){6}}

\put(3.5,60.5){\line(4,-3){4}}        
 \put(7.5,54.5){\framebox(3,3){}}     
\put(27.5,60.5){\line(4,-3){4}}       
 \put(31.5,54.5){\framebox(3,3){}}    
\put(51.5,60.5){\line(4,-3){4}}
 \put(55.5,54.5){\framebox(3,3){}}
\put(3.5,40.5){\line(4,-3){4}}
 \put(7.5,34.5){\framebox(3,3){}}
 \put(27.5,40.5){\line(4,-3){4}}
 \put(31.5,34.5){\framebox(3,3){}}
 \put(51.5,40.5){\line(4,-3){4}}
 \put(55.5,34.5){\framebox(3,3){}}
\put(3.5,20.5){\line(4,-3){4}}
 \put(7.5,14.5){\framebox(3,3){}}
 \put(27.5,20.5){\line(4,-3){4}}
 \put(31.5,14.5){\framebox(3,3){}}
 \put(51.5,20.5){\line(4,-3){4}}
 \put(55.5,14.5){\framebox(3,3){}}
\end{picture}
\vspace{-11ex}
\caption{\label{fig:2DP1}%
The factor graph representation of $f_\X(\x)$ in~(\ref{eqn:PDFP}) on a 2D lattice. The filled circles represent the
variables and the boxes represent the factors.
}
\vspace{1ex}
\setlength{\unitlength}{0.99mm}
\centering
\begin{picture}(62,78)(-6,-6)
\small
\put(0,60){\framebox(4,4){$=$}}
\put(12,60){\framebox(4,4){$+$}}
\put(15.95,61){$\circ$}
\put(17.38,62){\line(1,0){6.6}}
\put(24,60){\framebox(4,4){$=$}}
\put(28,62){\line(1,0){8}}         
\put(36,60){\framebox(4,4){$+$}}
\put(39.95,61){$\circ$}
\put(41.38,62){\line(1,0){6.6}}
\put(48,60){\framebox(4,4){$=$}}
\put(0,50){\framebox(4,4){$+$}}
\put(1.2,48.45){$\circ$}
\put(24,50){\framebox(4,4){$+$}}
\put(25.2,48.45){$\circ$}
\put(48,50){\framebox(4,4){$+$}}
\put(49.2,48.45){$\circ$}
\put(0,40){\framebox(4,4){$=$}}
\put(4,42){\line(1,0){8}}
\put(12,40){\framebox(4,4){$+$}}
\put(15.95,41.2){$\circ$}
\put(17.4,42){\line(1,0){6.6}}
\put(24,40){\framebox(4,4){$=$}}
\put(28,42){\line(1,0){8}}
\put(36,40){\framebox(4,4){$+$}}
\put(39.95,41.2){$\circ$}
\put(41.4,42){\line(1,0){6.6}}
\put(48,40){\framebox(4,4){$=$}}
\put(0,30){\framebox(4,4){$+$}}
\put(1.2,28.45){$\circ$}
\put(24,30){\framebox(4,4){$+$}}
\put(25.2,28.45){$\circ$}
\put(48,30){\framebox(4,4){$+$}}
\put(49.2,28.45){$\circ$}
\put(0,20){\framebox(4,4){$=$}}
\put(4,22){\line(1,0){8}}
\put(12,20){\framebox(4,4){$+$}}
\put(15.95,21){$\circ$}
\put(17.38,22){\line(1,0){6.6}}
\put(24,20){\framebox(4,4){$=$}}
\put(28,22){\line(1,0){8}}
\put(36,20){\framebox(4,4){$+$}}
\put(39.95,21){$\circ$}
\put(41.38,22){\line(1,0){6.6}}
\put(48,20){\framebox(4,4){$=$}}
\put(14,64){\line(0,1){2}}
\put(38,64){\line(0,1){2}}
\put(12,66){\framebox(4,4){}}   
\put(10,67.2){\pos{cb}{$\psi_1$}}
\put(-2.8,61.2){\pos{cb}{$\delta_{=}$}}
\put(7.35,62.4){\pos{cb}{$X_1$}}
\put(-1.25,55.5){\pos{cb}{$X^{''}_1$}}
\put(8.7,57.8){\pos{cb}{$X^{'}_1$}}
\put(21.4,62.4){\pos{cb}{$X_2$}}
\put(18.1,63.8){\pos{cb}{$Y_1$}}
\put(36,66){\framebox(4,4){}}   
\put(14,44){\line(0,1){2}}
\put(38,44){\line(0,1){2}}
\put(12,46){\framebox(4,4){$$}}
\put(36,46){\framebox(4,4){$$}}
\put(14,24){\line(0,1){2}}
\put(38,24){\line(0,1){2}}
\put(12,26){\framebox(4,4){$$}}
\put(36,26){\framebox(4,4){$$}}
\put(0,52){\line(-1,0){2}}
\put(24,52){\line(-1,0){2}}
\put(48,52){\line(-1,0){2}}
\put(-6,50){\framebox(4,4){$$}}
\put(18,50){\framebox(4,4){$$}}
\put(42,50){\framebox(4,4){$$}}
\put(0,32){\line(-1,0){2}}
\put(24,32){\line(-1,0){2}}
\put(48,32){\line(-1,0){2}}
\put(-6,30){\framebox(4,4){$$}}
\put(18,30){\framebox(4,4){$$}}
\put(42,30){\framebox(4,4){$$}}
\put(4,62){\line(1,0){8}}        

\put(28,62){\line(1,0){8}}       

\put(4,42){\line(1,0){8}}

\put(28,42){\line(1,0){8}}

\put(4,22){\line(1,0){8}}

\put(28,22){\line(1,0){8}}

\put(2,54){\line(0,1){6}}
\put(2,48.7){\line(0,-1){4.6}}
\put(26,54){\line(0,1){6}}
\put(26,48.7){\line(0,-1){4.6}}
\put(50,54){\line(0,1){6}}
\put(50,48.7){\line(0,-1){4.6}}
\put(2,34){\line(0,1){6}}
\put(2.0,28.7){\line(0,-1){4.6}}
\put(26,34){\line(0,1){6}}
\put(26.0,28.7){\line(0,-1){4.6}}
\put(50,34){\line(0,1){6}}
\put(50,28.7){\line(0,-1){4.6}}

\put(4,60){\line(4,-3){4}}        
 \put(8,54){\framebox(3,3){}}     
 \put(12,55.5){\pos{cl}{$\phi_1$}}
\put(28,60){\line(4,-3){4}}       
 \put(32,54){\framebox(3,3){}}    
\put(52,60){\line(4,-3){4}}
 \put(56,54){\framebox(3,3){}}
\put(4,40){\line(4,-3){4}}
 \put(8,34){\framebox(3,3){}}
 \put(28,40){\line(4,-3){4}}
 \put(32,34){\framebox(3,3){}}
 \put(52,40){\line(4,-3){4}}
 \put(56,34){\framebox(3,3){}}
\put(4,20){\line(4,-3){4}}
 \put(8,14){\framebox(3,3){}}
 \put(28,20){\line(4,-3){4}}
 \put(32,14){\framebox(3,3){}}
 \put(52,20){\line(4,-3){4}}
 \put(56,14){\framebox(3,3){}}
\end{picture}
\vspace{-12ex}
\caption{\label{fig:2DP2}%
The normal factor graph of $f_\X(\x)$ in~(\ref{eqn:PDFPr1}) on a 2D lattice. The small empty boxes
represent (\ref{eqn:PSIV1}), the big empty boxes represent (\ref{eqn:PSIV2}), boxes labeled $``="$ are equality indicator factors as in~(\ref{eqn:EQ}), and boxes labeled 
$``+"$ are zero-sum indicator factors given by (\ref{eqn:ZSum}). The symbol $``\circ"$ indicates a sign inversion.
}
\end{figure}



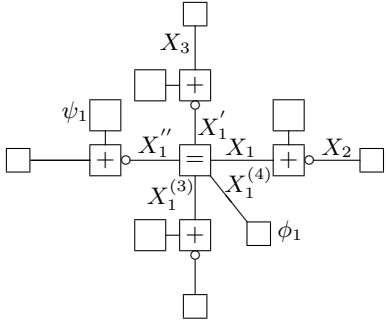
\begin{figure}[h]
\setlength{\unitlength}{0.99mm}
\centering
\begin{picture}(62,71)(-6,-6)
\small
\put(24.5,60){\framebox(3,3){$$}}
\put(24,50){\framebox(4,4){$+$}}
\put(25.2,48.6){$\circ$}
\put(0.85,40.5){\framebox(3,3){$$}}
\put(4,42){\line(1,0){8}}
\put(12,40){\framebox(4,4){$+$}}
\put(15.95,41.2){$\circ$}
\put(17.25,42){\line(1,0){6.75}}
\put(24,40){\framebox(4,4){$=$}}
\put(36.5,40){\framebox(4,4){$+$}}
\put(40.45,41.2){$\circ$}
\put(41.76,42){\line(1,0){6.32}}
\put(48.1,40.5){\framebox(3,3){$$}}
\put(24,30){\framebox(4,4){$+$}}
\put(25.2,28.6){$\circ$}
\put(24.5,21){\framebox(3,3){$$}}
\put(10,47){\pos{cb}{$\psi_1$}}
\put(45.0,42.4){\pos{cb}{$X_2$}}
\put(22.8,35.4){\pos{cb}{$X^{(3)}_1$}}
\put(28.4,44.8){\pos{cb}{$X^{'}_1$}}
\put(32.1,42.4){\pos{cb}{$X_1$}}
\put(20.8,42.4){\pos{cb}{$X^{''}_1$}}
\put(33.19,37.0){\pos{cb}{$X^{(4)}_1$}}
\put(23.4,56.3){\pos{cb}{$X_3$}}
\put(14,44){\line(0,1){2}}
\put(38.5,44){\line(0,1){2}}
\put(12,46){\framebox(4,4){$$}}
\put(36.5,46){\framebox(4,4){$$}}
\put(24,52){\line(-1,0){2}}
\put(18,50){\framebox(4,4){$$}}
\put(24,32){\line(-1,0){2}}
\put(18,30){\framebox(4,4){$$}}
\put(4,42){\line(1,0){8}}
\put(28,42){\line(1,0){8.5}}
\put(26,54){\line(0,1){6}}
\put(26,48.75){\line(0,-1){4.6}}
\put(26,34){\line(0,1){6}}
\put(26.0,28.75){\line(0,-1){4.68}}
\put(37,32.5){\pos{cl}{$\phi_1$}}
\put(28,40){\line(4,-5){5}}
\put(33,30.7){\framebox(3,3){}}
\end{picture}
\vspace{-17ex}
\caption{\label{fig:PStar}%
The normal factor graph of $f_\X(\x)$ in~(\ref{eqn:PDFPr1}) for a star graph $S_5$. Note that the graph is cycle-free. The 
precision matrix of the model is given in Example 2.
}
\end{figure}


\section{Gibbs Sampling}
\label{sec:GibbsP}

The Gibbs sampling algorithm (also known as the heat bath algorithm in statistical physics) is a Markov chain Monte Carlo 
method (mainly used in Bayesian statistics, image analysis, and machine learning) to 
generate samples according to a multivariate target distribution, here $f_\X(\cdot)$. 

The sampler chooses (randomly or deterministically) $x_i$, a coordinate of $\x$, then updates it by drawing a new sample according to 
the conditional distribution $f_\X(\cdot \cond x_j, j \ne i)$.

Among different updating strategies, we concentrate on the random-sweep Gibbs sampler. 
The sampler starts from an initial vector 
$\x^{(0)} = (x_1^{(0)}, x_2^{(0)}, \ldots, x_{|\V|}^{(0)})^\intercal$. Then, at iteration $\ell \ge 1$, the algorithm runs the following two steps $|\V|$ times: 
\begin{itemize}
\item randomly (i.e., uniformly and independently) choose an index $i$ from $\{1,2, \ldots, |\V|\}$.
\item update $x_i^{(\ell)}$ according to $f_\X(\,\cdot\,\cond x_j^{(\ell-1)}, j \ne i)$.
\end{itemize}

Under mild regularity conditions, the random-sweep Gibbs sampler creates a reversible Markov chain that
converges geometrically to the target distribution
$f_\X(\cdot)$. For more details, see~\cite{GG:84, liu2001monte, christian1999monte, murphy:2012}.

\subsection{convergence rates}
\label{sec:ConvP}

Let $\{Q_{ii}\}_{i=1}^{|\V|}$ denote the diagonal entries of $\Q$ in (\ref{eqn:QP2}), and define
\begin{equation}
\label{eqn:AGen}
\A = \I - \text{diag$(Q_{11}^{-1}, Q_{22}^{-1}, \ldots, Q_{|\V||\V|}^{-1})$}\Q
\end{equation}

The convergence rate of the random-sweep Gibbs sampler to a multivariate Gaussian target distribution is given by
\begin{equation}
\label{eqn:primalrate}
r_\text{p} = \big(|\V|^{-1}(|\V|-1 + \lambda_{\text{max}}(\A)\big)^{|\V|}
\end{equation}
For the derivation of (\ref{eqn:primalrate}), see~\cite{amit1996convergence}, \cite[Theorem 2]{roberts1997updating}. We observe that  $\A$ has zero 
diagonal entries; therefore, not all of its eigenvalues can be negative. In general, $\lambda_\mathrm{max}(\A)$ differs from $\rho(\A)=\max_i |\lambda_i(\A)|$. 

\begin{proposition}
\label{prop:Prop1}
The convergence rate of the random-sweep Gibbs sampler in $k$-regular homogeneous models is given by
\begin{equation}
r_\text{p} = \Big(1-\frac{\sigma^2}{\sigma^2+ks^2}\frac{1}{|\V|}\Big)^{|\V|}
\end{equation}
with 
\begin{equation}
\label{eqn:primalrate3}
\lim_{|\V| \to \infty} r_\text{p} = \textrm{exp}\Big(-\frac{\sigma^2}{\sigma^2+ks^2}\Big)
\end{equation}
as the asymptotic convergence rate.
\end{proposition}

\noindent
{\bf Proof.} The Laplacian matrix $\Lap = \B\B^\intercal$ is non-negative definite with $\text{rank$(\Lap)$} = |\V| -1$, and consequently $\Lap$ has an 
eigenvalue equal to zero with algebraic multiplicity one. We conclude that $\lambda_{\text{min}}(\Lap) = 0$. 

From~(\ref{eqn:QP2}), we obtain
\begin{equation}
\label{eqn:Qlambdamin}
\lambda_{\text{min}}(\Q) = \frac{1}{s^2}
\end{equation}

The diagonal entries of $\Q$ are equal to $1/s^2 + k/\sigma^2$. The largest eigenvalue of $\A$ in (\ref{eqn:AGen}) is thus
\begin{IEEEeqnarray}{rCl}
\lambda_{\text{max}}(\A) & = & 1 - \big(\frac{1}{s^2} + \frac{k}{\sigma^2}\big)^{-1}\frac{1}{s^2} \\[1mm]
				      & = & \frac{ks^2}{\sigma^2+ks^2} \label{eqn:primalrateHomogen}
\end{IEEEeqnarray}

Since the convergence rate only depends on $\lambda_{\text{max}}(\A)$, we obtain the desired convergence rate in closed-form after substituting (\ref{eqn:primalrateHomogen}) into (\ref{eqn:primalrate}). The asymptotic convergence rate 
is easily derived by considering the limit $|\V| \to \infty$.
\hfill $\blacksquare$

In a fully-connected graph $K_{|\V|}$, $k = |\V| -1$. It follows that
\begin{equation}
\lim_{|\V| \to \infty} r_\text{p} = 1
\end{equation}
and if $G$ is a balanced complete bipartite graph $K_{|\V/2|, |\V/2|}$, then $k=|\V|/2$. Thus
\begin{equation}
r_\text{p} = \Big(1-\frac{\sigma^2}{\sigma^2+s^2|\V|/2}\frac{1}{|\V|}\Big)^{|\V|}
\end{equation}
We will again obtain
\begin{equation}
\lim_{|\V| \to \infty} r_\text{p} = 1
\end{equation}

The derivation of the convergence rate for the homogeneous star graph $S_{|\V|}$ (as in Example 2) is provided 
in Appendix~\ref{appsec:StarG}.

\section{The Transformed (Dual) Model}
\label{sec:DualModel}

We examine the dual normal factor graph representation of the PDF in~(\ref{eqn:PDFPr1}). The dualization of normal graphs 
has been studied in the areas of error-correcting codes (such as codes on graphs)~\cite{RU:08}, algebraic topology~\cite{Forney:18}, and 
statistical physics~\cite{MoLo:ISIT2013}.

The dual normal graph is obtained by taking the Fourier transform of the local factors of a normal factor graph. The Fourier transform of a function $h(\x)$ is given by
\begin{IEEEeqnarray}{rCl}
\tilde{h}(\tilde{\x}) & =  & \FT(h)(\tilde{\x})\\
			    & = & \int_{-\infty}^{\infty} h(\x)\textrm{exp}(-\mathrm{i}\x\tilde{\x})d\x \label{eqn:FT}
\end{IEEEeqnarray}
where $\mathrm{i} = \sqrt{-1}$, see~\cite[Chapter 5]{stein2011fourier}.

The dual of the models that we study in this paper can be obtained by adopting the following steps:
\begin{itemize}
\item replace each variable $x_i$ by its dual variable $\tilde{x}_i$.
\item replace the factors $\phi(\cdot)$ and $\psi(\cdot)$ by their corresponding 1D Fourier transforms $\tilde{\phi}(\cdot)$ and $\tilde{\psi}(\cdot)$.
\item replace equality indicator factors by zero-sum indicator factors, and vice-versa.
\end{itemize}

From (\ref{eqn:FT}), we can obtain the 1D Fourier transform of the local factors (\ref{eqn:PSIV1}) and (\ref{eqn:PSIV2}) as
\begin{IEEEeqnarray}{rCl}
\tilde{\phi}(\tilde{x}_v)  & = & \sqrt{2\pi s^2}\,\textrm{exp}\Big(-\frac{s^2\tilde{x}_v^2}{2}\Big) \label{eqn:PSIDual1} \\[1mm]
\tilde{\psi}(\tilde{y}_e) & = & \sqrt{2\pi \sigma^2}\, \textrm{exp}\Big(-\frac{\sigma^2\tilde{y}_e^2}{2}\Big) \label{eqn:PSIDual2}
\end{IEEEeqnarray}
which are also Gaussian, up to scale. Hence the dual model remains Gaussian.

In the dual normal factor graph, $\tilde{\X}$ can be computed as a function of $\tilde{\Y}$. Indeed
\begin{equation}
\label{eqn:XfromY}
\tilde{\X} = \mathbf{B}\tilde{\Y}
\end{equation} 
where $\mathbf{B}$ is the incidence matrix described in Section~\ref{sec:Model}.
We will refer to $\tilde{\X}$ as auxiliary (or dependent) random variables~\cite{Forney:18}.

In the dual domain, the PDF is a function of $\tilde{\y}$, and is given by
\begin{equation}
\label{eqn:PDFDu3}
g_{\tilde{\Y}}(\tilde{\y}) = \frac{1}{Z_g}\prod_{v \in \V}\tilde{\phi}(\tilde{x}_v)\prod_{e\in \E}\tilde{\psi}(\tilde{y}_e) 
\end{equation}
where $Z_g$ is the appropriate normalization constant, and the marginal density over $\tilde{Y}_1$ is
\begin{equation}
\label{eqn:MargD} 
g_{\tilde{Y}_1}(\tilde{y}_1) = \int_{\tilde{y}_2}\int_{\tilde{y}_3}\ldots\int_{\tilde{y}_{|\E|}} g_{\tilde{\Y}}(\tilde{\y})d\tilde{y}_2d\tilde{y}_3\ldots d\tilde{y}_{|\E|}
\end{equation}

There are two important connections between a normal factor graph and its dual:
\begin{itemize}
\item the normalization constants $Z_f$ and $Z_g$ are equal up to a scale factor~\cite[Theorem 2]{AY:2011}.
\item the marginal densities $f_{X_i}(\cdot)$ and $g_{\tilde{Y}_i}(\cdot)$ are related by
\begin{equation}
\label{eqn:MargT1}
\frac{f_{X_i}(x)}{\phi_i(x)} = \FT\Big(\frac{g_{\tilde{Y}_i}}{\tilde{\phi}_i}\Big)(x)
\end{equation}
or equivalently 
\begin{equation}
\label{eqn:MargT2}
f_{X_i}(x) = \phi_i(x) \int_{-\infty}^{\infty} \frac{g_{\tilde{Y}_i}(\tilde{x})}{\tilde{\phi}_i(\tilde{x})}\textrm{exp}(-\mathrm{i}x\tilde{x})d\tilde{x}
\end{equation}
see~\cite[Proposition 1]{molkaraie2022mappings}.
\end{itemize}

Let the symmetric positive definite matrix $\mathbf{R} \in \R^{|\E|\times |\E|}$ denote the
precision matrix associated with the dual normal graph, i.e., with the PDF in (\ref{eqn:PDFDu3}). 
From (\ref{eqn:XfromY}) and (\ref{eqn:PDFDu3}) we obtain
\begin{IEEEeqnarray}{rCl}
g_{\tilde{\Y}}(\tilde{\y}) & \propto & \textrm{exp}\big(-\frac{s^2}{2}(\mathbf{B}\tilde{\y}^\intercal\mathbf{B}\tilde{\y}\big)\textrm{exp}\big(-\frac{\sigma^2}{2}\tilde{\y}^\intercal\tilde{\y}\big)  \\[1mm]
        & = & \textrm{exp}\big(-\frac{1}{2}\tilde{\y}^\intercal(s^2\B^\intercal\B + \sigma^2\I)\tilde{\y}\big) \label{eqn:PDFDu4} 
\end{IEEEeqnarray}
Therefore, $\mathbf{R}$ decomposes into 
\begin{equation}
\label{eqn:QD1}
\mathbf{R} = \sigma^2\I + s^2\B^\intercal\B 
\end{equation}
where $\mathbf{R}^{-1} = \text{Cov}(\tilde{\mathbf{Y}})$. The symmetric non-negative matrix $\B^\intercal\B$ is sometimes called the edge Laplacian of $G$.

In non-homogeneous models, suppose $\{1/s_v^2\}_{v=1}^{|\V|}$ and $\{1/\sigma_e^2\}_{e=1}^{|\E|}$ are the variances associated with 
the vertices and the edges, respectively. The precision matrix will be
\begin{equation}
\label{eqn:QPNHDual}
\mathbf{R} = \mathbf{D}_{\sigma} + \B^\intercal\mathbf{D}_{s}\B
\end{equation}
where diagonal matrices $\mathbf{D}_{\sigma}$ and $\mathbf{D}_{s}$ are as in~(\ref{eqn:QPNH}). 
We will again denote the local factors by $\tilde{\phi}_v(\cdot)$ and $\tilde{\psi}_e(\cdot)$ if the model is not homogeneous.

\begin{example} The dual of the normal factor graph in Fig.~\ref{fig:2DP1} is illustrated in Fig.~\ref{fig:2DP3}, where 
the small empty boxes represent (\ref{eqn:PSIDual1}) and the big empty boxes represent (\ref{eqn:PSIDual2}). The symbol $``\circ"$ denotes a sign inversion. E.g., the zero-sum indicator factor in Fig.~\ref{fig:2DP3} imposes
the constraint that $\tilde{X}_2 - \tilde{Y}_1 + \tilde{Y}_2 + \tilde{Y}_3= 0$. Here, $\tilde{\x}$ can be computed as a function of $\tilde{\y}$.

In homogeneous models, each variable, say $\tilde{y}_1$, appears in one factor $\tilde{\psi}(\cdot)$ attached to the corresponding equality indicator 
factor and to two zero-sum indicator factors attached to $\tilde{Y}_1$ because of pairwise interactions. We conclude that the diagonal entries of  
the precision matrix are $R_{ii}= \sigma^2+2s^2$ for $1 \le i \le |\E|$. 
\end{example}

\begin{example}
The dual of the normal factor graph of the star graph in Fig.~\ref{fig:PStar} can be easily obtained via the 
dualization procedure explained in Section~\ref{sec:DualModel}. In the homogeneous case, the PDF is 
\begin{equation*}
g_{\tilde{\Y}}(\tilde{\y}) \propto \textrm{exp}\Big(-\frac{\sigma^2 + 2s^2}{2}\sum_{e \in \E}\tilde{y}^2_e\Big)\textrm{exp}\Big(-s^2\sum_{e \ne e'}\tilde{y}_e\tilde{y}_{e'}\Big) 
\end{equation*}
Thus $R_{ii}= \sigma^2+2s^2$ for $1 \le i \le |\E|$.

\end{example}

\begin{figure}[t]
\setlength{\unitlength}{0.98mm}
\centering
\begin{picture}(62,78.2)(-6,-6)
\small
\put(0,60){\framebox(4,4){$+$}}
\put(12,60){\framebox(4,4){$=$}}
\put(15.95,61){$\circ$}
\put(17.38,62){\line(1,0){6.6}}
\put(24,60){\framebox(4,4){$+$}}
\put(28,62){\line(1,0){8}}         
\put(36,60){\framebox(4,4){$=$}}
\put(39.95,61){$\circ$}
\put(41.38,62){\line(1,0){6.6}}
\put(48,60){\framebox(4,4){$+$}}
\put(0,50){\framebox(4,4){$=$}}
\put(1.2,48.45){$\circ$}
\put(24,50){\framebox(4,4){$=$}}
\put(25.2,48.45){$\circ$}
\put(48,50){\framebox(4,4){$=$}}
\put(49.2,48.45){$\circ$}
\put(0,40){\framebox(4,4){$+$}}
\put(4,42){\line(1,0){8}}
\put(12,40){\framebox(4,4){$=$}}
\put(15.95,41.2){$\circ$}
\put(17.4,42){\line(1,0){6.6}}
\put(24,40){\framebox(4,4){$+$}}
\put(28,42){\line(1,0){8}}
\put(36,40){\framebox(4,4){$=$}}
\put(39.95,41.2){$\circ$}
\put(41.4,42){\line(1,0){6.6}}
\put(48,40){\framebox(4,4){$+$}}
\put(0,30){\framebox(4,4){$=$}}
\put(1.2,28.45){$\circ$}
\put(24,30){\framebox(4,4){$=$}}
\put(25.2,28.45){$\circ$}
\put(48,30){\framebox(4,4){$=$}}
\put(49.2,28.45){$\circ$}
\put(0,20){\framebox(4,4){$+$}}
\put(4,22){\line(1,0){8}}
\put(12,20){\framebox(4,4){$=$}}
\put(15.95,21){$\circ$}
\put(17.38,22){\line(1,0){6.6}}
\put(24,20){\framebox(4,4){$+$}}
\put(28,22){\line(1,0){8}}
\put(36,20){\framebox(4,4){$=$}}
\put(39.95,21){$\circ$}
\put(41.38,22){\line(1,0){6.6}}
\put(48,20){\framebox(4,4){$+$}}
\put(14,64){\line(0,1){2}}
\put(38,64){\line(0,1){2}}
\put(12,66){\framebox(4,4){}}   
\put(10,67.2){\pos{cb}{$\tilde{\psi}_{1}$}}
\put(36,66){\framebox(4,4){}}   
\put(14,44){\line(0,1){2}}
\put(38,44){\line(0,1){2}}
\put(12,46){\framebox(4,4){$$}}
\put(36,46){\framebox(4,4){$$}}
\put(14,24){\line(0,1){2}}
\put(38,24){\line(0,1){2}}
\put(12,26){\framebox(4,4){$$}}
\put(36,26){\framebox(4,4){$$}}
\put(0,52){\line(-1,0){2}}
\put(24,52){\line(-1,0){2}}
\put(48,52){\line(-1,0){2}}
\put(-6,50){\framebox(4,4){$$}}
\put(18,50){\framebox(4,4){$$}}
\put(42,50){\framebox(4,4){$$}}
\put(0,32){\line(-1,0){2}}
\put(24,32){\line(-1,0){2}}
\put(48,32){\line(-1,0){2}}
\put(-6,30){\framebox(4,4){$$}}
\put(18,30){\framebox(4,4){$$}}
\put(42,30){\framebox(4,4){$$}}
\put(4,62){\line(1,0){8}}        

\put(28,62){\line(1,0){8}}       

\put(4,42){\line(1,0){8}}

\put(28,42){\line(1,0){8}}

\put(4,22){\line(1,0){8}}

\put(28,22){\line(1,0){8}}

\put(2,54){\line(0,1){6}}
\put(2,48.7){\line(0,-1){4.6}}
\put(26,54){\line(0,1){6}}
\put(26,48.7){\line(0,-1){4.6}}
\put(50,54){\line(0,1){6}}
\put(50,48.7){\line(0,-1){4.6}}
\put(2,34){\line(0,1){6}}
\put(2.0,28.7){\line(0,-1){4.6}}
\put(26,34){\line(0,1){6}}
\put(26.0,28.7){\line(0,-1){4.6}}
\put(50,34){\line(0,1){6}}
\put(50,28.7){\line(0,-1){4.6}}
\put(18.5,63.8){\pos{cb}{$\tilde{Y}_1$}}
\put(42.5,63.8){\pos{cb}{$\tilde{Y}_2$}}
\put(23.2,54.7){\pos{cb}{$\tilde{Y}_3$}}
\put(8.9,57.9){\pos{cb}{$\tilde{X}_1$}}
\put(32.9,57.9){\pos{cb}{$\tilde{X}_2$}}
\put(26.5,65.2){\pos{cb}{$\delta_{+}$}}
\put(4,60){\line(4,-3){4}}        
 \put(8,54){\framebox(3,3){}}     
 \put(12,55){\pos{cl}{$\tilde{\phi}_1$}}
\put(28,60){\line(4,-3){4}}       
 \put(32,54){\framebox(3,3){}}    
\put(52,60){\line(4,-3){4}}
 \put(56,54){\framebox(3,3){}}
\put(4,40){\line(4,-3){4}}
 \put(8,34){\framebox(3,3){}}
 \put(28,40){\line(4,-3){4}}
 \put(32,34){\framebox(3,3){}}
 \put(52,40){\line(4,-3){4}}
 \put(56,34){\framebox(3,3){}}
\put(4,20){\line(4,-3){4}}
 \put(8,14){\framebox(3,3){}}
 \put(28,20){\line(4,-3){4}}
 \put(32,14){\framebox(3,3){}}
 \put(52,20){\line(4,-3){4}}
 \put(56,14){\framebox(3,3){}}
\end{picture}
\vspace{-11ex}
\caption{\label{fig:2DP3}%
The dual of the normal factor graph in Fig.~\ref{fig:2DP2} that represents the factorization in~(\ref{eqn:PDFDu3}). The small empty boxes
represent (\ref{eqn:PSIDual1}) and the big empty boxes represent (\ref{eqn:PSIDual2}).
}
\end{figure}
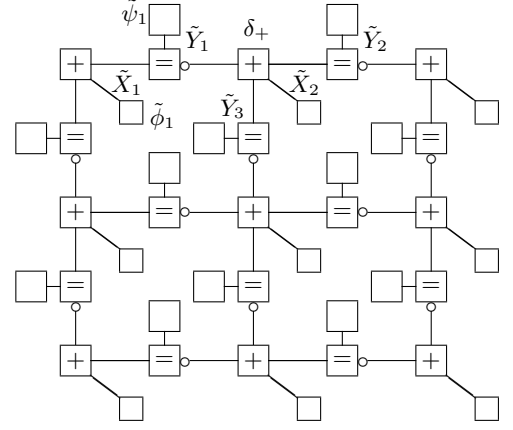

The next proposition establishes an algebraic connection between covariance matrices in the primal and dual domains. Consequently, one may estimate primal covariance 
statistics entirely from samples generated in the dual model.

\begin{proposition}
\label{prop:Prop2}
The covariance matrix in the primal domain $\text{Cov}(\X)$ and the covariance matrix of the auxiliary variables $\tilde{\X}$ in the 
dual domain $\text{Cov}(\tilde{\X})$ are related by 
\begin{equation}
\label{eq:cov_woodbury}
\mathbf{D}_s^{-1/2}\text{Cov}(\X)\mathbf{D}_s^{-1/2} + \mathbf{D}^{1/2}_s \text{Cov}(\tilde{\X}) \mathbf{D}^{1/2}_s = \mathbf{I}
\end{equation}
where $\tilde{\X} = \mathbf{B}\tilde{\Y}$.
\end{proposition}

\noindent
{\bf Proof.} From (\ref{eqn:QPNH}), we can compute $\text{Cov}(\X)$ as
\begin{IEEEeqnarray}{rCl}
\text{Cov}(\X) & = & \mathbf{Q}^{-1} \\
& = & (\mathbf{D}_{s}^{-1} + \B\mathbf{D}_{\sigma}^{-1}\B^\intercal)^{-1}
\end{IEEEeqnarray}

Applying the Woodbury matrix identity\footnote{The identity states that
the inverse of the matrix \mbox{$\mathbf{A}+ \mathbf{U}\mathbf{C}\mathbf{V}$} is equal to $\mathbf{A}^{-1} - \mathbf{A}^{-1}\mathbf{U}(\mathbf{C}^{-1} +\mathbf{V}\mathbf{A}^{-1}\mathbf{U})^{-1}\mathbf{V}\mathbf{A}^{-1}$; see Sherman-Morrison-Woodbury formula in~\cite{horn2013matrix}.} yields
\begin{IEEEeqnarray}{rCl}
    \text{Cov}(\X) &= & \mathbf{D}_s - \mathbf{D}_s \mathbf{B}(\mathbf{D}_\sigma + \mathbf{B}^\intercal\mathbf{D}_s \mathbf{B})^{-1} \mathbf{B}^\intercal\mathbf{D}_s \\
    &= & \mathbf{D}_s - \mathbf{D}_s \mathbf{B}\text{Cov}(\tilde{\mathbf{Y}}) \mathbf{B}^\intercal \mathbf{D}_s \label{eq:cov_woodbury2} 
\end{IEEEeqnarray}
where (\ref{eq:cov_woodbury2}) follows from~\eqref{eqn:QPNHDual}. 

As $\tilde{\X} = \mathbf{B}\tilde{\Y}$, it holds that $\text{Cov}(\tilde{\X}) = \mathbf{B}\text{Cov}(\tilde{\mathbf{Y}})\mathbf{B}^\intercal$. Thus 
\begin{equation}
\label{eq:cov_woodbury3}
\text{Cov}(\X) = \mathbf{D}_s - \mathbf{D}_s \text{Cov}(\tilde{\X}) \mathbf{D}_s 
\end{equation}
The result follows from premultiplying and postmultiplying both sides of (\ref{eq:cov_woodbury3}) by $\mathbf{D}_s^{-1/2}$.   \hfill $\blacksquare$

Intuitively, Proposition 2 indicates that the weighted total variance is conserved, in the sense that low variance (i.e., high certainty) in the 
primal domain corresponds to high variance (i.e., low certainty) in the dual domain. Indeed, from (\ref{eq:cov_woodbury}) we obtain 
\begin{equation}
\label{eq:cov_woodbury5}
\frac{1}{s_v^2}\text{Var$(X_v)$} + s_v^2\text{Var$(\tilde{X}_v)$} = 1
\end{equation}
for $1 \le v \le |\V|$. 

Here \eqref{eq:cov_woodbury5} can be viewed as a form of statistical uncertainty principle, demonstrating that the system shares a fixed variance budget. 
Furthermore, from a practical standpoint, \eqref{eq:cov_woodbury5} allows us to estimate the variance of $X_v$ from the variance of $\tilde{X}_v$, which in turn can be estimated directly from 
samples drawn in the dual domain using Gibbs sampling; see Section~\ref{sec:GibbsDual}.

From (\ref{eq:cov_woodbury5}), it is straightforward to show that
\begin{IEEEeqnarray}{rCl}
\label{eq:ExactvarBounds}
\text{Var$(X_v)$} & \le & s_v^2 \\
\text{Var$(\tilde{X}_v)$} & \le & \frac{1}{s_v^2}
\end{IEEEeqnarray}
and by applying the AM--GM inequality to (\ref{eq:cov_woodbury5}), we obtain
\begin{equation}
\label{eq:uncertainly}
\text{Var$(X_v)$}\text{Var$(\tilde{X}_v)$} \le \frac{1}{4}
\end{equation}
with equality if and only if $s_v^{-2}\text{Var$(X_v)$} = s_v^{2}\text{Var$(\tilde{X}_v)$} = \frac{1}{2}$.

%

\begin{proposition}
\label{prop:Prop3}
The covariance matrix in the dual domain $\text{Cov}(\tilde{\Y})$ and the covariance matrix of the 
auxiliary variables $\Y$ in the primal domain $\text{Cov}(\Y)$  are related by
\begin{equation}
\label{eq:cov_woodbury7}
\mathbf{D}^{1/2}_\sigma \text{Cov}(\tilde{\Y}) \mathbf{D}^{1/2}_\sigma + \mathbf{D}_\sigma^{-1/2}\text{Cov}(\Y)\mathbf{D}_\sigma^{-1/2} = \mathbf{I}
\end{equation}
where $\Y = \mathbf{B}^\intercal\X$. 
\end{proposition}

The proof of Proposition~\ref{prop:Prop3} is analogous to that of Proposition~\ref{prop:Prop2} and is omitted.

Using the same reasoning
\begin{equation}
\label{eq:cov_woodbury88}
\frac{1}{\sigma_e^2}\text{Var$(Y_e)$} + \sigma_e^2\text{Var$(\tilde{Y}_e)$} = 1
\end{equation}
for $1 \le e \le |\E|$, which yields
\begin{IEEEeqnarray}{rCl}
\label{eq:ExactvarBoundsForY}
\text{Var$(Y_e)$} & \le & \sigma_e^2 \\
\text{Var$(\tilde{Y}_e)$} & \le & \frac{1}{\sigma_e^2}
\end{IEEEeqnarray}
and
\begin{equation}
\label{eq:uncertainlyD}
\text{Var$(Y_e)$}\text{Var$(\tilde{Y}_e)$} \le \frac{1}{4}
\end{equation}
with equality if and only if $\sigma_e^{-2}\text{Var$(Y_e)$} = \sigma_e^{2}\text{Var$(\tilde{Y}_e)$} = 1/2$. This result should be compared to (\ref{eq:uncertainly}).

Derivations of lower and upper bounds on $\text{det}(\text{Cov}(\X))$, $H(\X)$, and $H(\tilde{\Y})$ are provided in Appendix~\ref{appsec:EntropyBound}.

\section{Gibbs Sampling in the Dual Domain}
\label{sec:GibbsDual}

The random-sweep Gibbs sampler of Section~\ref{sec:GibbsP} can be employed to draw samples 
according to $g_{\tilde{\Y}}(\tilde{\y})$ in (\ref{eqn:PDFDu3}). We start from an initial vector 
$\tilde{\y}^{(0)} = (\tilde{y}_1^{(0)}, \tilde{y}_2^{(0)}, \ldots, \tilde{y}_{|\E|}^{(0)})^\intercal$. At iteration $\ell \ge 1$, the Gibbs sampler runs the following steps $|\E|$ times: 
\begin{enumerate}[i)]
\item randomly choose an index $i$ from $\{1,2, \ldots, |\E|\}$.
\item update $\tilde{y}_i^{(\ell)}$ according to $g_{\tilde{\Y}}(\,\cdot\,\cond \tilde{y}_j^{(\ell-1)}, j \ne i)$.
\end{enumerate}

We can derive the convergence rate of the sampler to a multivariate Gaussian target distribution from the results of Section~\ref{sec:ConvP}. Indeed, in the dual domain 
\begin{equation}
\label{eqn:dualrate}
r_\text{d} = \big(|\E|^{-1}(|\E|-1 + \lambda_{\text{max}}(\A)\big)^{|\E|}
\end{equation}
where
\begin{equation}
\label{eqn:AD}
\A = \I - \text{diag$(R_{11}^{-1}, R_{22}^{-1}, \ldots, R_{|\E||\E|}^{-1})$}\mathbf{R}
\end{equation}
and $\{R_{ii}\}_{i=1}^{|\E|}$ are the diagonal entries of $\mathbf{R}$ in (\ref{eqn:QPNHDual}).

\begin{proposition}
\label{prop:Prop4}
In the dual of homogeneous models with cycles, the convergence rate of the random-sweep Gibbs sampler is
\begin{equation}
\label{eqn:dualrate2}
r_\text{d} = \Big(1-\frac{\sigma^2}{\sigma^2+2s^2}\frac{1}{|\E|}\Big)^{|\E|}
\end{equation}
with
\begin{equation}
\label{eqn:dualrate3}
\lim_{|\E| \to \infty} r_\text{d} = \textrm{exp}\Big(-\frac{\sigma^2}{\sigma^2+2s^2}\Big)
\end{equation}
as the asymptotic convergence rate.
\end{proposition}

\noindent
{\bf Proof.} From (\ref{eqn:QD1}), $\mathbf{R} = \sigma^2\I + s^2\B^\intercal\B$. 
The two matrices $\B^\intercal\B \in \R^{|\E|\times |\E|}$ and $\B\B^\intercal \in \R^{|\V|\times |\V|}$ share the same 
rank and the same $|\V| -1$ non-zero eigenvalues; see~Section~\ref{sec:Model}. Therefore, $\B^\intercal\B$ has $|\E| - |\V| + 1$ zero 
eigenvalues. If $|\E| > |\V| - 1$ (i.e., if $G$ has cycles), then $\lambda_{\text{min}}(\B^\intercal\B) = 0$.

Hence
\begin{equation}
\label{eqn:Rlambdamin}
\lambda_{\text{min}}(\mathbf{R}) = \sigma^2
\end{equation}
Since the diagonal entries of $\mathbf{R}$ are equal to $\sigma^2 + 2s^2$, the largest eigenvalue of $\A$ in (\ref{eqn:AD}) is
\begin{equation}
\label{eqn:primalrateHomogenD}
\lambda_{\text{max}}(\A) = \frac{2s^2}{\sigma^2+2s^2}
\end{equation}

The result follows from substituting (\ref{eqn:primalrateHomogenD}) into (\ref{eqn:dualrate}). The asymptotic convergence rate is obtained 
by considering $|\E| \to \infty$. \hfill $\blacksquare$

We emphasize that the rates in (\ref{eqn:dualrate2}) and (\ref{eqn:dualrate3}) remain valid for both 
fully-connected and balanced complete bipartite graphs. Indeed, these rates apply to all homogeneous graphs with PDF in (\ref{eqn:PDFP}) satisfying the 
condition $|\E| > |\V| - 1$, which only excludes cycle-free graphical models. As a result, under the 
framework of the models studied in this paper, the rates in the dual domain achieve \emph{universality} across all graph structures with cycles.

A comparison of the rates in (\ref{eqn:primalrate3}) and (\ref{eqn:dualrate3}) 
shows that a faster asymptotic convergence rate is achieved in the dual domain when $k > 2$. We next demonstrate that this rate can be further improved.

\subsection{The Effective Convergence Rate}

Equation \eqref{eqn:dualrate3} gives the convergence rate for the edge
vector $\tilde{\Y}$. In practice, however, the object of interest is usually not
$\tilde{\Y}$; but its image $\tilde{\X}=\B\tilde{\Y}$. 

For the oriented incidence matrix $\B \in \R^{|\V|\times |\E|}$, the null space $\mathcal{N}(\B)$ is equal to the cycle space of $G$, which has dimension
$|\E|-|\V|+1$. The row space $\mathcal{R}(\B^\intercal)$ coincides with the cut space of $G$, and has dimension $|\V|-1$. 
Indeed
\begin{equation}
\label{eqn:cyclecut}
\mathbb{R}^{|\E|} = \mathcal{N}(\B) \oplus \mathcal{R}(\B^\intercal)
\end{equation}
where $\oplus$ denotes the direct sum \cite{horn2013matrix}. Any $\tilde{\Y}$ can be written as the sum $\tilde{\Y}_{\text{cycle}} + \tilde{\Y}_{\text{cut}}$ of its cycle space and
its cut space components. Furthermore 
\begin{IEEEeqnarray}{rCl}
\tilde{\X} & = &\B(\tilde{\Y}_{\text{cycle}} + \tilde{\Y}_{\text{cut}}) \label{eq:XYB}\\
 & =  & \B\tilde{\Y}_{\text{cut}} \label{eq:XYB1}
\end{IEEEeqnarray}
which means $\tilde{\X}$ depends exclusively on the cut-space component of $\tilde{\Y}$, whereas its cycle-space component is completely annihilated.

We will show that
working directly with $\tilde{\X}$ yields a strictly faster convergence rate. The rate is determined by $\lambda_2(\Lap)$ the smallest nonzero
eigenvalue of $\Lap$  (i.e., the algebraic connectivity of $G$).

\begin{proposition}
\label{prop:Prop5}
In a homogeneous model defined on a graph $G$ with cycles, let
$\lambda_2(\Lap)$ denote the algebraic connectivity of $G$. Any statistic that
depends on $\tilde{\Y}$ only through $\tilde{\X}=\B\tilde{\Y}$ can be estimated by
the random-sweep Gibbs sampler in the dual domain at the effective rate
\begin{equation}
\label{eqn:effectiverate1}
r^{\text{eff}}_\text{d}
   = \Big(1-\frac{\sigma^2+s^2\lambda_2(\Lap)}{\sigma^2+2s^2}\frac{1}{|\E|}\Big)^{|\E|}
\end{equation}
with
\begin{equation}
\label{eq:r_eff_final}
\lim_{|\E|\to\infty} r^{\text{eff}}_\text{d}
   = \exp\!\left(-\frac{\sigma^2+s^2\lambda_2(\Lap)}{\sigma^2+2s^2}\right)
\end{equation}
as the asymptotic convergence rate.
\end{proposition}

\noindent
{\bf Proof.}
In homogeneous models, $R_{ee}=\sigma^2+2s^2$ for every $e \in \E$. From (\ref{eqn:QD1}), we can express $\A$ in
\eqref{eqn:AD} as
\begin{IEEEeqnarray}{rCl}
\A & = & \I - \frac{1}{\sigma^2+2s^2}(\sigma^2\I + s^2\B^\intercal\B ) \\
 & =  & \frac{s^2}{\sigma^2+2s^2}\bigl(2\I-\B^\intercal\B\bigr)
\end{IEEEeqnarray}

As $\B^\intercal\B$ and the Laplacian $\Lap=\B\B^\intercal$ share the same
nonzero eigenvalues, the spectrum of $\A$ consists of
\begin{itemize}
\setlength{\itemsep}{1.6ex}
\item $\kappa = \dfrac{2s^2}{\sigma^2+2s^2}$, with algebraic multiplicity
$|\E|-|\V|+1$, whose eigenvectors span the cycle space $\mathcal N(\B)$.
\item $\mu_i = \dfrac{s^2(2-\lambda_i(\Lap))}{\sigma^2+2s^2}$,
for $i=2,\dots,|\V|$. Here $0<\lambda_2(\Lap)\le\cdots\le\lambda_{|\V|}(\Lap)$
are the nonzero eigenvalues of $\Lap$. The corresponding eigenvectors span the cut space
$\mathcal R(\B^\intercal)$.
\end{itemize}

By spectral decomposition $\A = \mathbf P \boldsymbol\Lambda \mathbf P^\intercal$, where
$\boldsymbol\Lambda=\mathrm{diag}(\kappa,\dots,\kappa,\mu_2,\dots,\mu_{|\V|})$,
the first $|\E|-|\V|+1$ columns of $\mathbf{P}$ consist of the eigenvectors that span the cycle space, and the
remaining $|\V|-1$ columns of $\mathbf{P}$ are the eigenvectors that span the cut space.

Following \cite[Theorems~2 and 3]{roberts1997updating}, 
the spectral analysis of the Gaussian random-sweep sampler shows that its dynamics can be viewed as a mixture of autoregressive processes. The 
convergence rate of any linear functional of the state is governed by the eigenvalues of $\A$ in the eigenspaces where the functional has a nonzero 
projection. Consequently, the deviation of the state $\tilde{\Y}$
from stationarity can be decomposed in the eigenbasis of $\A$. 
The component along each eigenvector decays geometrically at a rate obtained from substituting the
corresponding eigenvalue into \eqref{eqn:dualrate}. As a result, the convergence rate of any statistic of $\tilde{\Y}$ is determined by the largest 
eigenvalue among the eigendirections on which the statistic has a nonzero projection.

From \eqref{eq:XYB1}, $\tilde{\X} = \B\tilde{\Y}_{\text{cut}}$. Therefore $\tilde{\X}$ and every statistic computed 
from $\tilde{\X}$ (e.g., $\mathrm{Cov}(\X)$ and the marginal variances)
depend on $\tilde{\Y}$ only through the eigencomponents in the cut-space. The convergence
rate of any such statistic is thus governed by the maximum eigenvalue in $\mathcal{R}(\B^\intercal)$. Since $\mu_i$ is decreasing in $\lambda_i(\Lap)$, this maximum is
attained at $\mu_2$. Indeed
\begin{equation}
\max_{\mathbf{v}\in\mathcal{R}(\B^\intercal)}\frac{\mathbf{v}^\intercal \A\mathbf{v}}{\mathbf{v}^\intercal \mathbf{v}}=\max_{i\ge2}\mu_i = \mu_2 
\end{equation}
where
\begin{equation}
\mu_2 = \frac{s^2(2-\lambda_2(\Lap))}{\sigma^2+2s^2}
\end{equation}
Substituting $\mu_2$ into \eqref{eqn:dualrate} gives \eqref{eqn:effectiverate1}, and letting $|\E|\to\infty$ yields \eqref{eq:r_eff_final}. 
\hfill $\blacksquare$

For $k$-regular graphs with $k>2$, a comparison of the asymptotic rates in
\eqref{eqn:primalrate3} and \eqref{eqn:dualrate3} further shows that
\begin{equation}
\label{eqn:RateComparison}
\lim_{|\E| \to \infty}  r^{\text{eff}}_\text{d}<\lim_{|\E|\to\infty} r_\text{d} \;<\; \lim_{|\V|\to\infty} r_\text{p}.
\end{equation}
For any finite graph $\mu_2<\kappa$, because $\lambda_2(\Lap) > 0$. Therefore, $r^{\text{eff}}_\text{d} < r_\text{d}$ holds in general. 

\begin{table*}[t!]
\centering
\small
\begin{tabular}{c|c|c}
Graph & $\lambda_2(\mathbf{L})$ & Asymptotic $r^{\text{eff}}_\text{d}$ as $s/\sigma\to\infty$ \\
\midrule
Complete graph $K_{|\V|}$
    & $|\V|$
    & $\exp(-|\V|/2)$ \\[4pt]
Balanced complete bipartite graph $K_{n,n}$ $(|\V|=2n)$
    & $n$
    & $\exp(-|\V|/4)$ \\[4pt]
$n\times n$ torus $(|\V| = n^2)$
    & $4\sin^2(\pi/n)$
    & $\exp\bigl(-2\sin^2(\pi/n)\bigr)$ \\[4pt]
Random $k$-regular graph (large $|\V|$) \cite{friedman2004proofalonssecondeigenvalue}
    & $k-2\sqrt{k-1}+o(1)$
    & $\exp\bigl(-(k-2\sqrt{k-1})/2\bigr)$ \\[4pt]
Star graph $S_{|\V|}$
    & $1$
    & $\exp(-1/2)$ \\
\end{tabular}
\caption{
$\lambda_2(\mathbf{L})$ and the asymptotic effective convergence rates for several common graph families. }
\label{tab:fiedler}
\end{table*}


According to Propositions~\ref{prop:Prop1} and~\ref{prop:Prop4}, in the primal and dual domains of homogeneous $k$-regular models, the asymptotic convergence rates 
tend to 1 as $s/\sigma \to \infty$. However, according to Proposition~\ref{prop:Prop5}, $r^{\text{eff}}_\text{d}$ asymptotically approaches $\textrm{exp}(-\lambda_2(\Lap)/2)$, which is strictly less than 1. 
The improvement is governed by the algebraic connectivity $\lambda_2(\Lap)$. Table~\ref{tab:fiedler}
lists $\lambda_2(\Lap)$ for several common graph families. 


The convergence rates for the homogeneous star graph (as a cycle-free model described in Examples 2 and 4) are derived 
in Appendix~\ref{appsec:StarG}, where we will show that the effective convergence rate matches the rate reported in Table~\ref{tab:fiedler}.


\subsection{Computational Complexity}
\label{sec:Complexity}

In Sections~\ref{sec:GibbsP} and~\ref{sec:GibbsDual}, we proved that significantly improved convergence rates can be achieved in the dual domain. However, it is important to ensure 
that these gains are not offset by a higher computational cost per iteration. 

The coordinate updates of the Gibbs sampler for a multivariate Gaussian distribution with precision matrix $\mathbf{M}$ can be expressed through the following reparameterization
\begin{IEEEeqnarray}{rCl}
    z_i^{(\ell)} &= & \mu_i^{(\ell)} + \sqrt{M_{ii}^{-1}} \epsilon^{(\ell)}, \quad \epsilon^{(\ell)} \sim \mathcal{N}(0, 1) \\
    \mu_i^{(\ell)} &= & -M_{ii}^{-1} \sum_{j \neq i} M_{ij} z_j^{(\ell-1)} \label{eqn:sampling_general}
\end{IEEEeqnarray}

In the primal domain, $\mathbf{M} = \mathbf{Q}$ as in (\ref{eqn:QPNH}), and updating each node requires a sum over all its neighbors. As the number 
of nonzero off-diagonal entries of $\mathbf{Q}$ is $2|\E|$, a full sweep will require $\mathcal{O}(|\mathcal{E}|)$ operations. 

In the dual domain, $\mathbf{M} = \mathbf{R}$ as in (\ref{eqn:QPNHDual}), which is usually dense. A naive computation of (\ref{eqn:sampling_general}) requires a sum over all the edges 
incident to the endpoints of the target edge $e = (u,v)$, which results in a cost proportional to $(d_u + d_v)$ per update and a total sweep complexity of $\mathcal{O}(\sum_{v \in \mathcal{V}} d_v^2)$. 

However, by 
maintaining a running tally of the auxiliary variables $\tilde{\mathbf{x}} = \mathbf{B}\tilde{\mathbf{y}}$, we can reduce the 
complexity to $\mathcal{O}(1)$ per coordinate (i.e., per edge) update. 

More specifically, since $R_{ee'} = (\mathbf{B}^\top \mathbf{D}_s \mathbf{B})_{ee'}$ for $e\ne e'$, we can rewrite 
the sum in (\ref{eqn:sampling_general}) as
\begin{IEEEeqnarray}{rCl}
    \sum_{e' \neq e} R_{ee'} \tilde y_{e'} 
    &=& \sum_{e'\neq e} \sum_{k\in\mathcal{V}} s_k^2 B_{ke} B_{ke'} \tilde{y}_{e'} \\
    &= & \sum_{k\in\mathcal{V}}s_k^2B_{ke}\big(\sum_{e'\neq e}B_{ke'}\tilde{y}_{e'}\big)\\
    &= & \sum_{k\in\mathcal{V}}s_k^2B_{ke}\big((\mathbf{B}\tilde{\mathbf{y}})_k - B_{ke}\tilde{y}_e\big)      
\end{IEEEeqnarray}
Thus
\begin{IEEEeqnarray}{rCl}
    \sum_{e' \neq e} R_{ee'} \tilde y_{e'} &= &\sum_{k\in\mathcal{V}} s_k^2 B_{ke} \left( \tilde{x}_k - B_{ke} \tilde{y}_e \right) \\
    & = & s_u^2 (\tilde x_u - \tilde y_e) - s_v^2 (\tilde x_v + \tilde y_e) \label{eqn:sampling_simplified}
\end{IEEEeqnarray}

where \eqref{eqn:sampling_simplified} follows from the definition of the incidence matrix $\mathbf{B}$ for the edge $e=(u,v)$; see Section~\ref{sec:Model}. 

Substituting (\ref{eqn:sampling_simplified}) into (\ref{eqn:sampling_general}) allows for the evaluation of $\tilde{\mu}_e$ in constant time, followed by the 
update $\tilde{y}_e \leftarrow \tilde{y}_e^{\text{new}}$. We can guarantee consistency among vertex variables $\tilde{\mathbf{x}}$ using incremental updates 
\begin{IEEEeqnarray}{rCl}
\tilde{x}_u & \leftarrow & \tilde{x}_u + \Delta \\
\tilde{x}_v & \leftarrow & \tilde{x}_v - \Delta
\end{IEEEeqnarray}
with $\Delta = \tilde{y}_e^{\text{new}} - \tilde{y}_e^{\text{old}}$. This state maintenance ensures that the complexity of a full sweep in 
the dual domain is $\mathcal{O}(|\mathcal{E}|)$, which matches the complexity of the primal domain.

\section{Numerical Experiments}
\label{sec:Num}
We present numerical experiments to support the theoretical
properties of the random-sweep Gibbs sampler. The experiments provide empirical verification for 
the identity (\ref{eq:cov_woodbury}) in Proposition~\ref{prop:Prop2}, for the faster convergence of the sampler in the dual domain than in the primal domain 
predicted in Propositions~\ref{prop:Prop1} and~\ref{prop:Prop4}, for the universality of the convergence rate in the dual domain implied by Proposition~\ref{prop:Prop4}, 
and for the effective convergence rate established in Proposition~\ref{prop:Prop5}. We further show that the dual Gibbs sampler retains its advantage beyond the setting covered 
by our theoretical analysis, including on heterogeneous, non-regular graphs and under alternative updating strategies. We also demonstrate its scalability on graphs too large for the 
full covariance matrix to be computed exactly and to be stored.

We compute matrix norms in terms of the Frobenius norm $\|\cdot\|_\mathrm{F}$, where
\begin{equation}
\|\A\|^2_\mathrm{F} = \operatorname{trace}(\A^\intercal \A)
\end{equation}
for $\A \in \R^{m\times n}$. For more details, see~\cite[Chapter 5]{horn2013matrix}.

\subsection{Verification of Proposition~\ref{prop:Prop2}}
\label{subsec:verification}

We empirically verify the variance-conservation identity (\ref{eq:cov_woodbury}) on a
homogeneous $8\times8$ torus with $s=1.0$ and $\sigma=0.3$, using
$10^5$ samples in the dual domain. Fig.~\ref{fig:stochastic_pythagorean}(a) shows the analytically computed
$\mathbf{D}_s^{-1/2}\operatorname{Cov}(\mathbf{X})\mathbf{D}_s^{-1/2}$ and Fig.~\ref{fig:stochastic_pythagorean}(b) shows the estimated
$\mathbf{D}_s^{1/2}\operatorname{Cov}(\tilde{\mathbf{X}})\mathbf{D}_s^{1/2}$ from
samples drawn in the dual domain. According to (\ref{eq:cov_woodbury}), their sum should be indistinguishable from the identity matrix. 

Fig.~\ref{fig:stochastic_pythagorean}(c) illustrates the residual
$\boldsymbol{\Delta} = \mathbf{D}_s^{-1/2}\operatorname{Cov}(\mathbf{X})\mathbf{D}_s^{-1/2}
+\mathbf{D}_s^{1/2}\operatorname{Cov}(\tilde{\mathbf{X}})\mathbf{D}_s^{1/2}-\mathbf{I}$,
whose entries are visibly an order of magnitude smaller than the entires in panels (a) and (b),
confirming that the sum closely matches the identity across the whole matrix, rather than merely on average.

Fig.~\ref{fig:stochastic_pythagorean}(d) shows the
distribution of the normalized Frobenius residual $\|\boldsymbol{\Delta}\|_\mathrm{F}/\sqrt{|\mathcal{V}|}$
obtained from $10^4$ bootstrap replicates. Its small mean and narrow spread
confirm that the identity holds to high accuracy. These results demonstrate that the dual sampler captures the full information content of 
the model, allowing primal statistics to be reconstructed from dual samples alone.

\begin{figure*}[t!]
\centering
\begin{subfigure}[b]{0.4\linewidth}
\centering
\includegraphics[width=\linewidth]{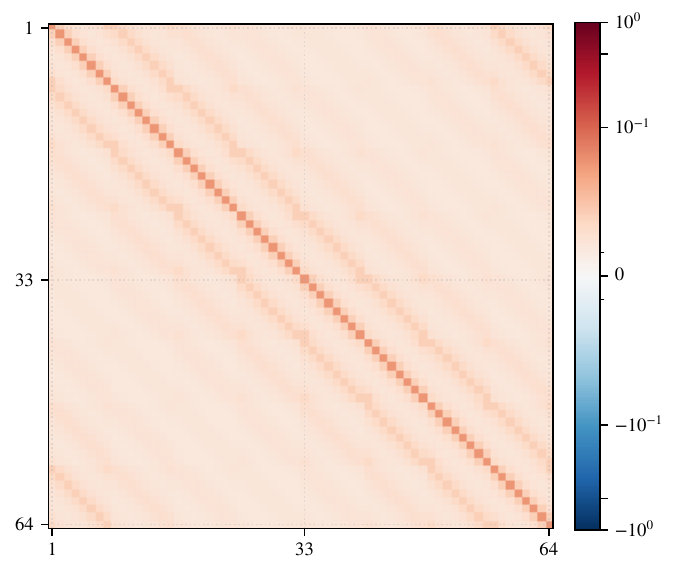}
\caption{}
\label{fig:woodbury_primal}
\end{subfigure}\hfill
\begin{subfigure}[b]{0.4\linewidth}
\centering
\includegraphics[width=\linewidth]{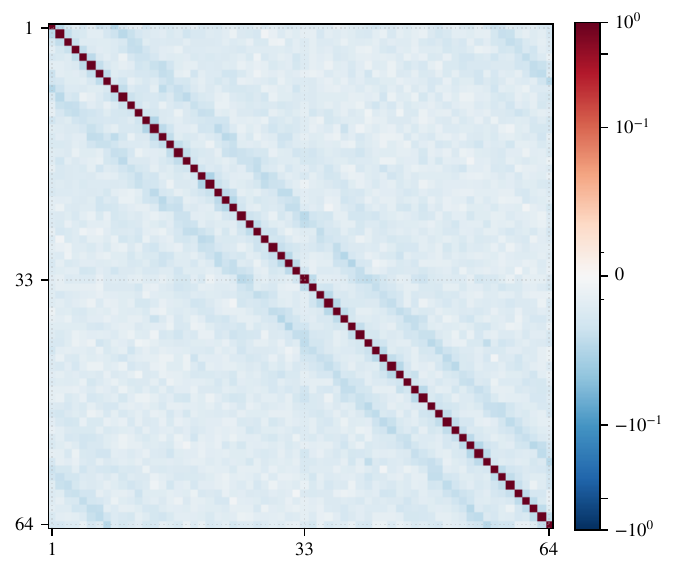}
\caption{}
\label{fig:woodbury_dual}
\end{subfigure}

\vspace{0.6em}

\begin{subfigure}[b]{0.4\linewidth}
\centering
\includegraphics[width=\linewidth]{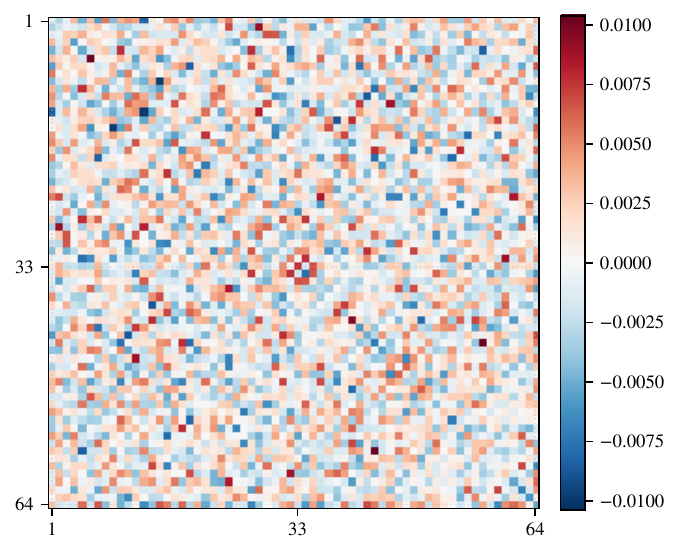}
\caption{}
\label{fig:woodbury_residual}
\end{subfigure}\hfill
\begin{subfigure}[b]{0.4\linewidth}
\centering
\includegraphics[width=\linewidth]{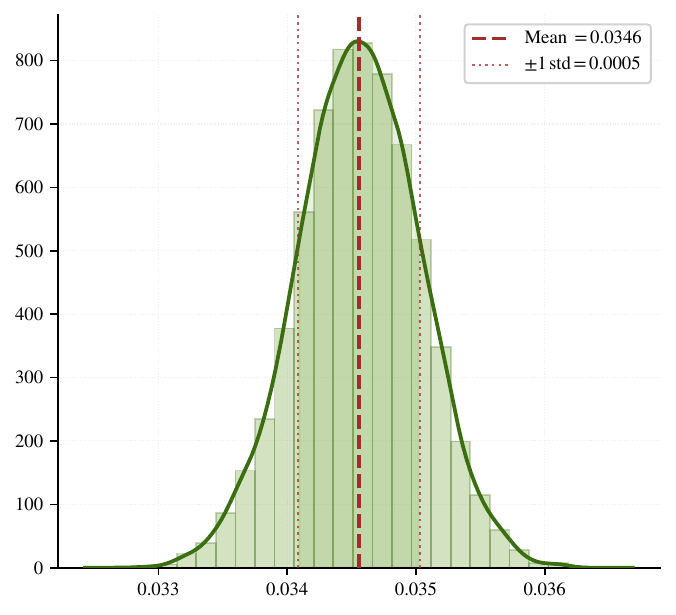}
\caption{}
\label{fig:woodbury_dist}
\end{subfigure}
\caption{Empirical verification of the identity in~(\ref{eq:cov_woodbury})
on a homogeneous $8\times8$ torus with $s=1.0$ and $\sigma=0.3$, using
$10^5$ samples in the dual domain:
(a)~Analytically computed $\mathbf{D}_s^{-1/2}\operatorname{Cov}(\mathbf{X})\mathbf{D}_s^{-1/2}$
(b)~Estimated $\mathbf{D}_s^{1/2}\operatorname{Cov}(\tilde{\mathbf{X}})\mathbf{D}_s^{1/2}$
(c)~Residual $\boldsymbol{\Delta}$ defined as the difference between the sum of (a) and (b) and the identity matrix
(d)~Error distribution of $\|\boldsymbol{\Delta}\|_\mathrm{F}/\sqrt{|\mathcal{V}|}$ obtained from bootstrapping
over $10^4$ replicates.}
\label{fig:stochastic_pythagorean}
\end{figure*}

\subsection{Convergence Analysis}
\label{subsec:convergence}

We compare convergence of the random-sweep Gibbs sampler in the primal and dual domains for
estimating the covariance matrix $\boldsymbol{\Sigma}$ (i.e., $\operatorname{Cov}(\X)$).
A naive plug-in estimator $\|\hat{\boldsymbol{\Sigma}}-\boldsymbol{\Sigma}\|_\mathrm{F}^2$ is
positively biased by the finite-sample variance of $\hat{\boldsymbol{\Sigma}}$.
As a result, the error reaches a nonzero floor, obscuring the geometric decay predicted by our theory. 
We therefore employ an unbiased estimator instead.

In each domain, we first run $L = 10^4$ independent trials. At each sweep $t$, we estimate the covariance matrix from
the values obtained from each trial. In the dual domain, we apply Proposition~\ref{prop:Prop2}
to compute an estimate of the primal covariance matrix using
$\operatorname{Cov}(\mathbf{X})=\mathbf{D}_s-\mathbf{D}_s\operatorname{Cov}(\tilde{\mathbf{X}})\mathbf{D}_s$ in (\ref{eq:cov_woodbury3}).
In this setting, we ensure that computations in both domains are scored against the same target $\boldsymbol{\Sigma}$.

Let $\boldsymbol{\Sigma}_t$ denote the true covariance of the chain's marginal law
at sweep $t$. We consider the squared Frobenius
error to $\boldsymbol{\Sigma}$, i.e., $\|\boldsymbol\Sigma_t-\boldsymbol\Sigma\|_\mathrm{F}^2$. We split the $L$ trials into two independent halves, each yielding the estimates $\hat{\boldsymbol{\Sigma}}_1$ and
$\hat{\boldsymbol{\Sigma}}_2$. We then form the unbiased estimator
\begin{equation}
\label{eq:unbiased_error}
    \hat{\mathcal{E}}
    = \operatorname{trace}\bigl(
        (\hat{\boldsymbol{\Sigma}}_1-\boldsymbol{\Sigma})^\intercal
        (\hat{\boldsymbol{\Sigma}}_2-\boldsymbol{\Sigma})
      \bigr).
\end{equation}
Since $\hat{\boldsymbol{\Sigma}}_1$ and $\hat{\boldsymbol{\Sigma}}_2$ are independent
and $\mathbb{E}[\hat{\boldsymbol{\Sigma}}_1] = \mathbb{E}[\hat{\boldsymbol{\Sigma}}_2] = \boldsymbol{\Sigma}_t$, it follows that
\begin{equation}
\mathbb{E}[\hat{\mathcal{E}}]=\|\boldsymbol{\Sigma}_t-\boldsymbol{\Sigma}\|_\mathrm{F}^2
\end{equation}
For more details, see~\cite{Owen2013, LehmannCasella1998}.

Because $\hat{\mathcal{E}}$ is a cross product  between two independent deviations
rather than a squared norm, it may take negative values. Such values merely indicate 
that the chain has reached $\boldsymbol{\Sigma}$ to within sampling error.

\begin{figure}[t!]
\centering
\includegraphics[width=0.94\linewidth]{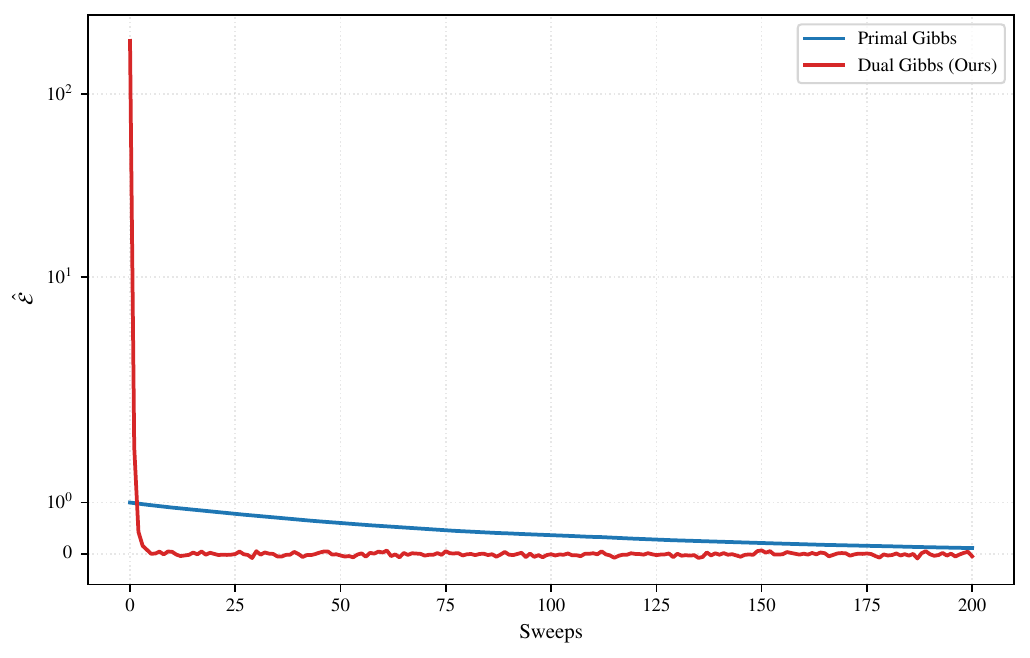}
\caption{Convergence on a homogeneous $10\times10$ torus with $s=1$ and $\sigma=0.1$, using
$L=10^4$ trials. The plot shows $\hat{\mathcal{E}}$ (normalized to
start at $1$) versus the number of sweeps.}
\label{fig:mse_convergence_sweeps}
\end{figure}

Fig.~\ref{fig:mse_convergence_sweeps} shows $\hat{\mathcal{E}}$ as a function of the number of sweeps on a homogeneous
$10\times10$ torus with $s=1$ and $\sigma=0.1$. We observe that the dual $\hat{\mathcal{E}}$
drops by orders of magnitude within a few sweeps before fluctuating around
zero, whereas the primal $\hat{\mathcal{E}}$ decays slowly throughout the entire run.

\begin{figure}[t!]
\centering
\includegraphics[width=0.94\linewidth]{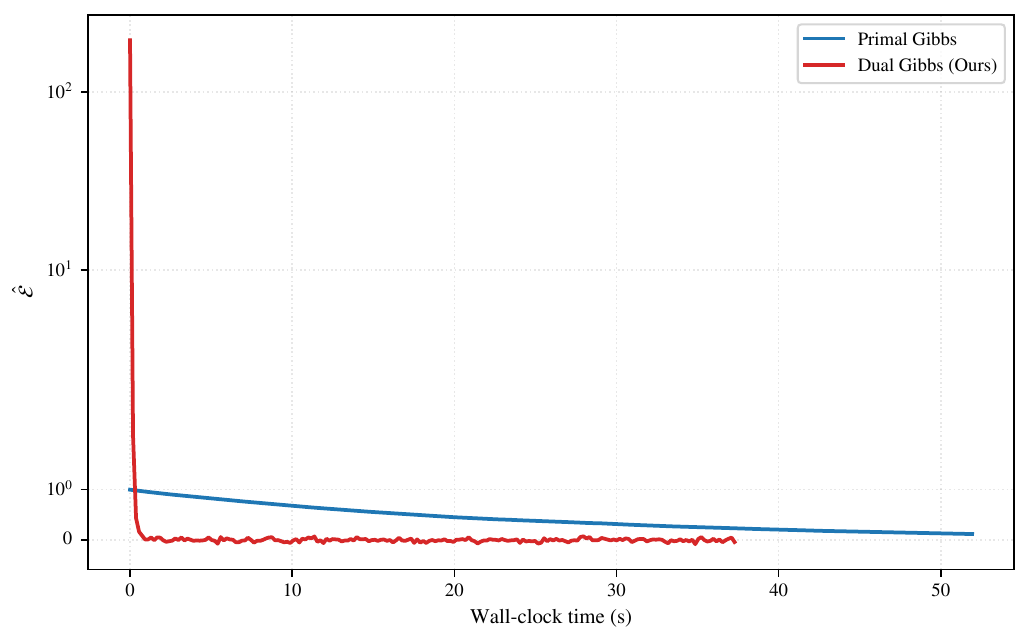}
\caption{Convergence on a homogeneous $10\times10$ torus with $s=1$ and $\sigma=0.1$, using
$L=10^4$ trials. The plot shows $\hat{\mathcal{E}}$ (normalized to
start at $1$) versus the wall-clock time in seconds.}
\label{fig:mse_convergence_time}
\end{figure}

Fig.~\ref{fig:mse_convergence_time} plots the same quantity against 
wall-clock time. Since the per-sweep cost is essentially identical in the two domains,
the number of sweeps serves as a faithful proxy for the computational cost. As shown in Section~\ref{sec:Complexity},
the computational complexity is $\mathcal{O}(|\mathcal{E}|)$ per sweep for both samplers.

Since the theory in Sections~\ref{sec:GibbsP}--\ref{sec:GibbsDual} characterizes the
rate of convergence rather than the initial error, we normalize each curve by its value at 
sweep $0$. Therefore, the comparisons isolate
the geometric decay rather than arbitrary scaling introduced by the initialization. Normalized 
errors are denoted by $\hat{\mathcal{E}}\text{(norm.)}$ in subsequent figures.


To confirm that the observed advantage is not an artifact of a particular graph realization, Fig.~\ref{fig:nonhomog_band} shows the same comparison on heterogeneous,
non-regular Watts--Strogatz (WS) graphs\footnote{A Watts--Strogatz graph is
constructed by starting from a ring lattice in which each node connects to its
$k$ nearest neighbors, and then independently rewiring each edge to a
uniformly random endpoint with probability $p$, see~\cite{watts1998collective}.}.
We generate $R=64$ independent realizations by sampling a new graph $\mathrm{WS}(|\mathcal{V}|=64,k=4,p=0.3)$ together with independent vertex and edge variances
$s_v \overset{\text{iid}}{\sim} \mathcal{U}(0.8, 1.2)$ and
$\sigma_e \overset{\text{iid}}{\sim} \mathcal{U}(0.2, 0.3)$.

Fig.~\ref{fig:nonhomog_band} reports the empirical mean of $\hat{\mathcal{E}}$ computed across realizations with its $\pm1$
standard deviation band.
The results show that the dual sampler consistently converges well ahead of the primal across the entire ensemble.

Our theoretical results are derived for the random-sweep Gibbs sampler; Fig.~\ref{fig:scan_comparison} confirms that the dual advantage persists empirically under alternative 
updating strategies, including random-permutation and deterministic sweeps, on the same $10\times10$ torus.

\begin{figure}[t!]
\centering
\includegraphics[width=0.94\linewidth]{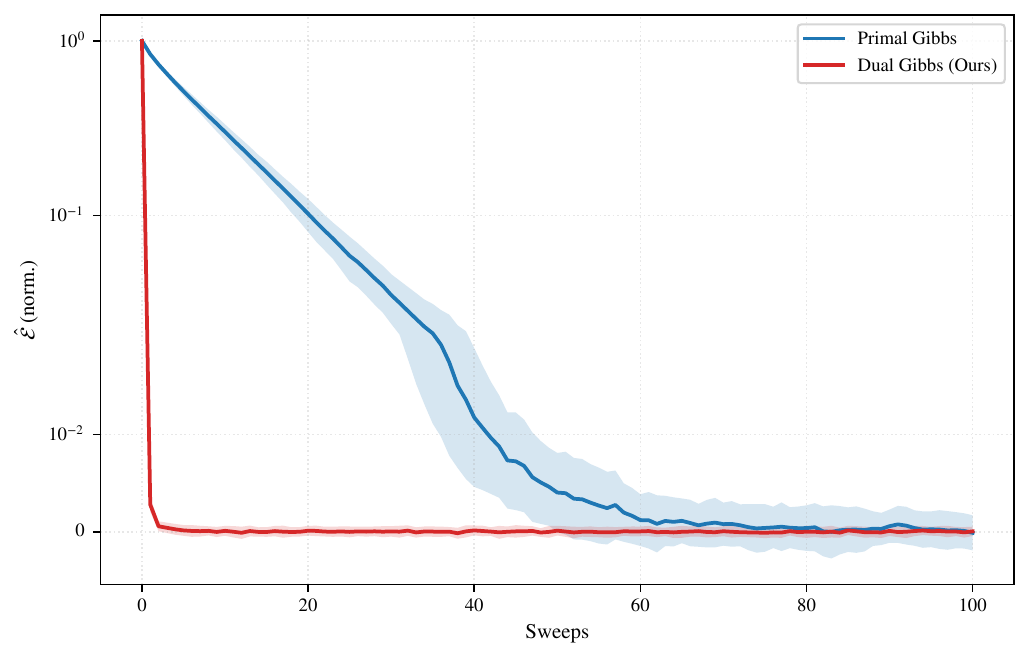}
\caption{Convergence on heterogeneous, non-regular Watts--Strogatz graphs over
$R=64$ independent realizations. The plot shows the empirical mean of $\hat{\mathcal{E}}$ across realizations within a shaded $\pm1$
standard deviation band. The dual advantage holds across the ensemble.}
\label{fig:nonhomog_band}
\end{figure}

\subsection{Robustness to Graph Structure and to Model Parameters}
\label{subsec:robustness}

We next investigate the effect of the node degree $k$ and the coupling strength $s/\sigma$ on convergence. Fig.~\ref{fig:k_robustness} shows the convergence results 
for $k$-regular homogeneous graphs with $|\mathcal{V}|=64$, $s=1.0$, and $\sigma=0.25$, for several values of $k$. 

In the primal domain (solid lines),
convergence deteriorates as $k$ increases, in agreement with Proposition~\ref{prop:Prop1}. 
The worst-case convergence rate in the dual domain given by (\ref{eqn:dualrate2}) (dashed lines) is independent of $k$, consistent with Proposition~\ref{prop:Prop4}. 
Moreover, Proposition~\ref{prop:Prop5} predicts that the effective rate (\ref{eqn:effectiverate1}), which governs $\tilde{\X}$, improves as $k$ increases because denser graphs have larger 
algebraic connectivity $\lambda_2(\mathbf{L})$. 

Since the ordering of the curves is only visible during the first few sweeps, before the dual estimator reaches the sampling-noise floor, the inset highlights this early stage.

\begin{figure}[t!]
\centering
\includegraphics[width=0.94\linewidth]{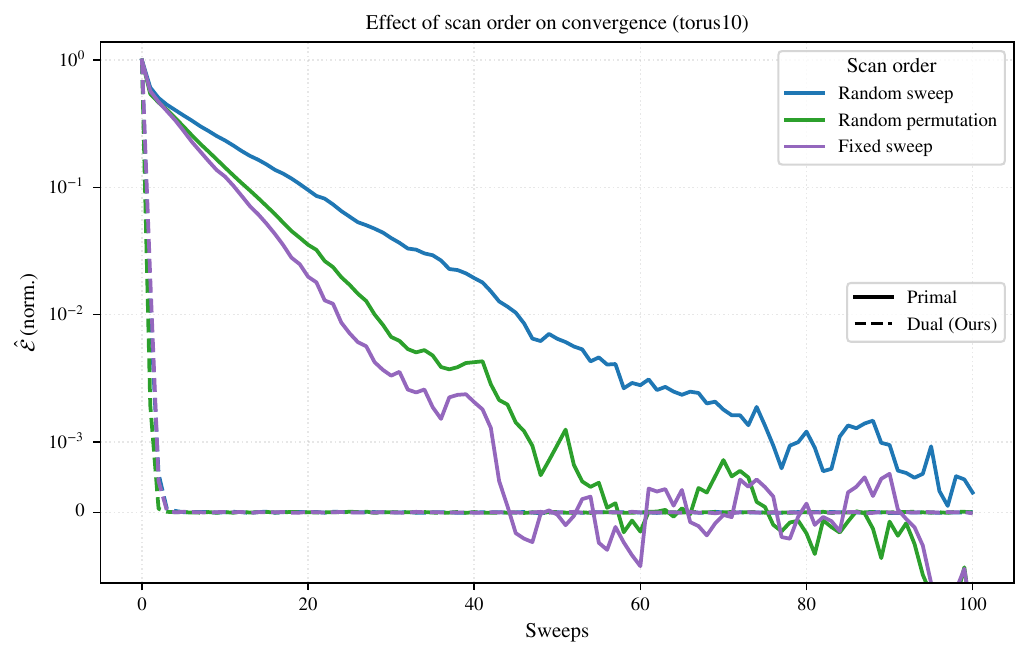}
\caption{Convergence under different updating strategies (random-sweep, random-permutation, and deterministic-sweep) on a homogeneous $10\times10$ torus with $s=1$ 
and $\sigma=0.1$. The dual advantage persists across all three 
strategies.}
\label{fig:scan_comparison}
\end{figure}

\begin{figure}[t!]
\centering
\includegraphics[width=0.94\linewidth]{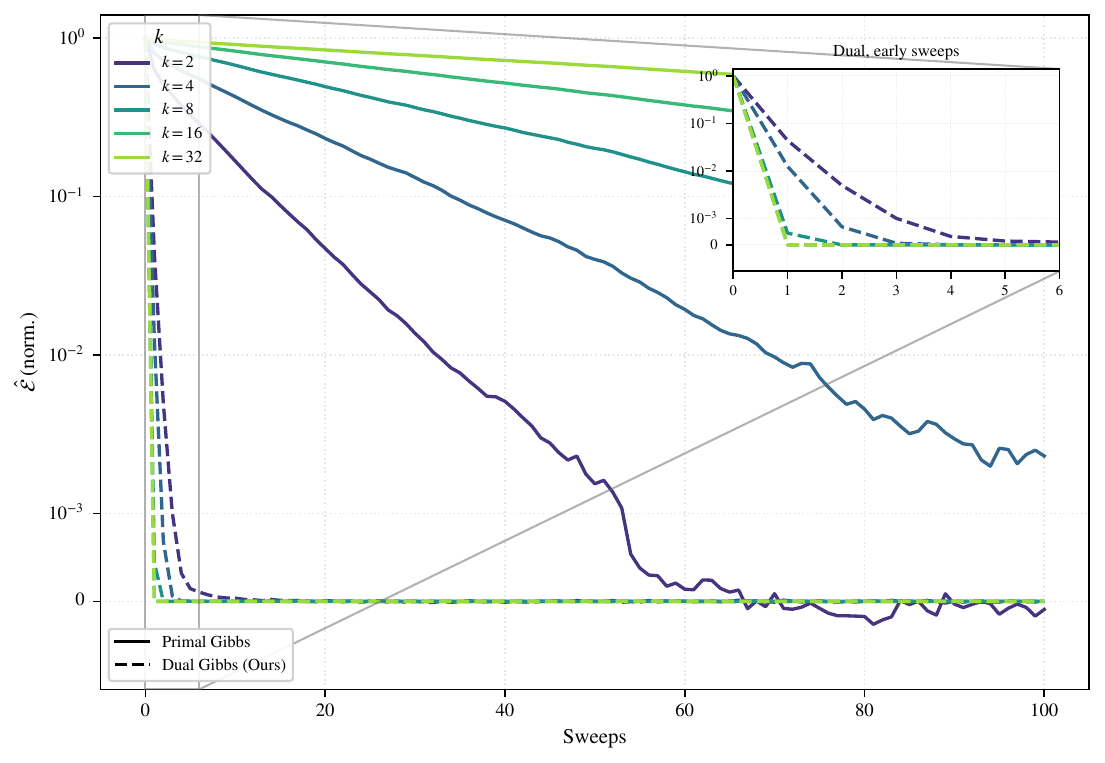}
\caption{Effect of degree $k$ on convergence for $k$-regular graphs with
$|\mathcal{V}|=64$, $s=1.0$, and $\sigma=0.25$, using the squared-error
estimator $\hat{\mathcal{E}}$. Convergence in the primal domain (solid lines)
degrades as $k$ grows. In contrast, convergence in the dual domain (dashed lines) improves
as $k$ increases. The inset shows the same dual curves over the first few sweeps, where their ordering is more visible.}
\label{fig:k_robustness}
\end{figure}

Fig.~\ref{fig:coupling_robustness} reports the negative logarithm of the convergence rate (i.e., $-\ln r$), estimated from the decay of $\hat{\mathcal{E}}$ over 
the regime in which it exhibits geometric decay as
a function of $s/\sigma$ on a logarithmic scale for the fully connected graph $K_5$.
As $s/\sigma$ grows, the primal convergence rate approaches one, causing its logarithm to collapse toward zero, which explains its slow mixing in the strong-coupling 
regime (see Proposition~\ref{prop:Prop1}). The dual effective rate approaches $\lambda_2(\mathbf{L})/2$ with $\lambda_2(\mathbf{L}) = 5$ (the red dotted plot), in quantitative agreement with
Proposition~\ref{prop:Prop5} and Table~\ref{tab:fiedler}.

\subsection{Scalability to Large Graphs}
\label{subsec:scalability}

While the previous experiments verify convergence of the full covariance matrix
$\boldsymbol{\Sigma}$, its exact computation requires the inversion of a precision matrix with computational complexity $\mathcal{O}(|\mathcal{V}|^3)$, making it computationally 
prohibitive for large graphs. To demonstrate that the dual advantage
holds at scale, we instead track the marginal variances (i.e., the diagonal entries of $\boldsymbol{\Sigma}$). This approach reduces the memory requirement 
to $\mathcal{O}(|\mathcal{V}|)$ per sweep in either domain.

\begin{figure}[t!]
\centering
\includegraphics[width=0.94\linewidth]{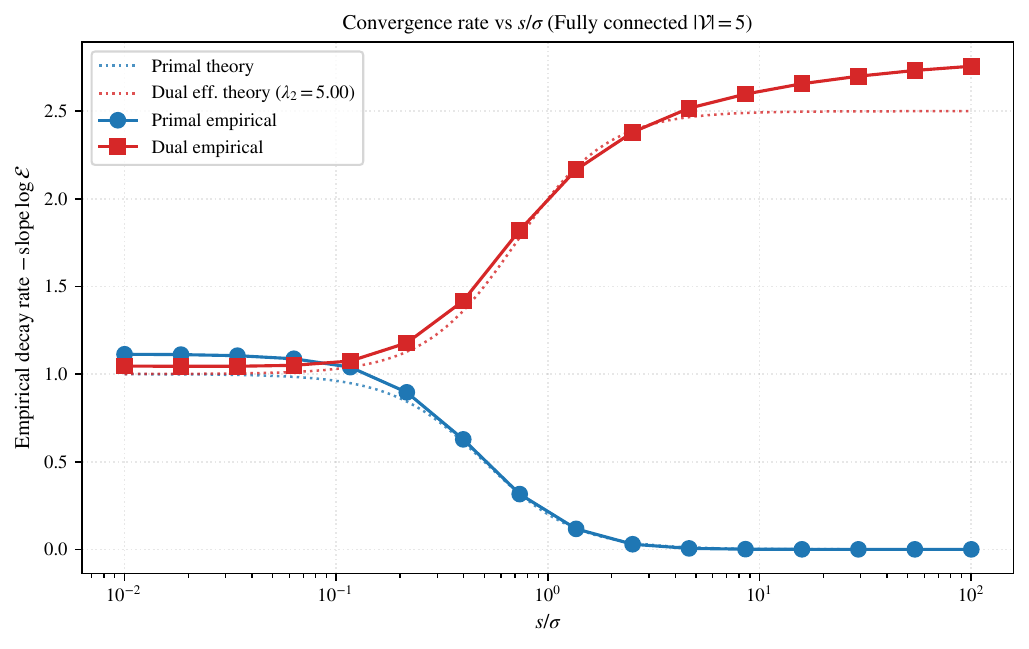}
\caption{Empirical decay rate $-\ln r$ versus the coupling strength $s/\sigma$ on a logarithmic scale for the fully connected graph $K_5$. The primal rate collapses as
$s/\sigma\to\infty$, while the dual effective rate stays bounded away from zero and
approaches $\lambda_2(\mathbf{L})/2$.}
\label{fig:coupling_robustness}
\end{figure}

By symmetry in an $n\times n$ homogeneous torus, every node shares the same
marginal variance, which admits the following closed form
\begin{equation}
\label{eq:torus_true_var}
\nu^\star = \frac{1}{n^2}\sum_{a,b=0}^{n-1}
\left(\frac{1}{s^2}+\frac{\lambda_{a,b}}{\sigma^2}\right)^{-1}
\end{equation}
with
\begin{equation}
\label{eq:torus_true_var2}
\lambda_{a,b} = 4-2\cos\!\Big(\frac{2\pi a}{n}\Big)-2\cos\!\Big(\frac{2\pi b}{n}\Big)
\end{equation}
This analytical solution, which is derived from the eigenvalues of the Laplacian matrix $\mathbf{L}$, serves as the ground truth against which
samplers in both domains are evaluated.

We consider a homogeneous $100\times100$ torus with $s=1.0$ and $\sigma=0.25$.
Across $L=600$ independent trials, we employ 
$\hat{\mathcal{E}}$ in~(\ref{eq:unbiased_error}) to estimate the marginal variances. These diagonal entries of
$\boldsymbol{\Sigma}$ are all equal to $\nu$.
While the primal sampler converges considerably slower on this large-scale model, the dual sampler reaches the analytical ground truth within a 
few sweeps. 
This confirms that the dual advantage extends beyond small graphs.

\section{Code Availability}
The implementation used in this work, together with scripts for reproducing all numerical experiments and figures, is 
available online at \url{https://github.com/Theborna/dual_gibbs}


\section{Future Work}
\label{sec:Future}

\begin{figure}[t!]
\centering
\includegraphics[width=0.94\linewidth]{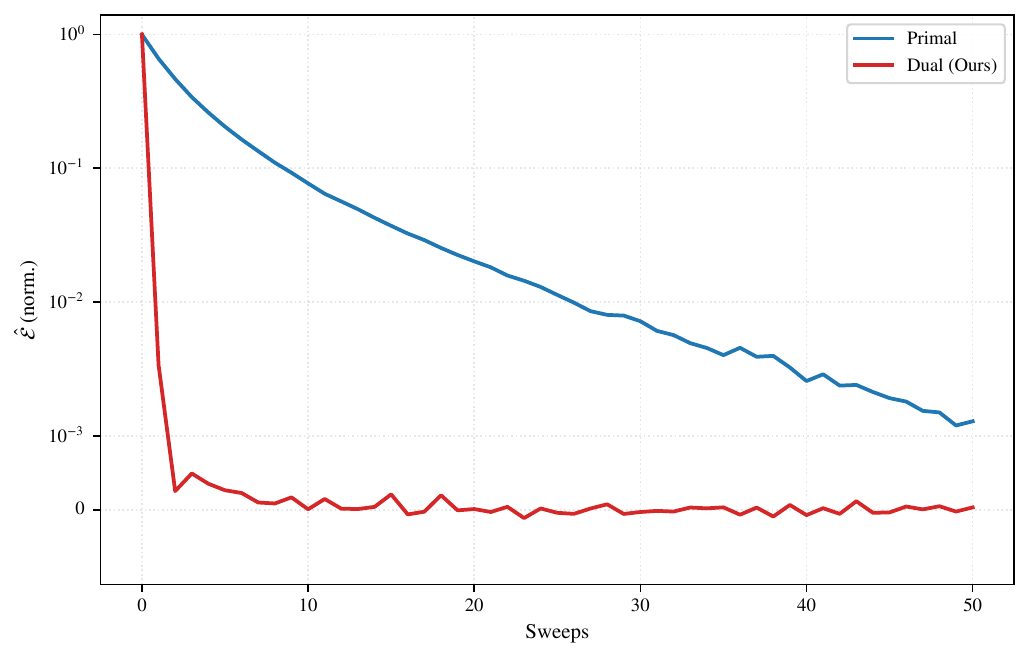}
\caption{Convergence of the estimated error in marginal variances to the analytical ground truth~(\ref{eq:torus_true_var}) on a $100\times100$
torus with $s=1.0$ and $\sigma=0.25$, computed across $L=600$ trials. The dual sampler reaches the ground truth within a few sweeps, whereas the primal sampler converges markedly slower.}
\label{fig:large_torus}
\end{figure}

There are several immediate extensions of the results presented in this paper, including:

\begin{itemize}

\item The choice of updating strategy can have a significant effect on the convergence rate of the Gibbs sampler. While this paper 
focuses on the random-sweep Gibbs sampler (except for numerical experiments in Fig.~\ref{fig:scan_comparison}), alternative updating schemes (e.g., random-permutation or deterministic-sweep) can 
also be applied to both primal and dual models. However, their convergence behavior is typically more difficult to analyze than that of random-sweep Gibbs 
sampling; see\cite{liu1995covariance}, \cite[Section 2]{roberts1997updating}. A rigorous comparison of these updating schemes is 
an interesting direction for future research.

\item As a generalization of the model in~(\ref{eqn:PDFP}), we can place bivariate Gaussian 
factors $\psi_e(\x) = \textrm{exp}(-\frac{1}{2}\x^\intercal\Q_e\x)$ on each edge. In 
the dual normal factor graph, each factor is replaced by its 2D Fourier transform, which is equal 
to $\tilde{\psi}_e(\tilde{\x}) = \textrm{exp}\big(-\frac{1}{2}\tilde{\x}^\intercal\Q_e^{-1}\tilde{\x}\big)$, up 
to scale. Extending the convergence analysis to this more general class of Gaussian graphical models seems more challenging because it 
involves additional model parameters (e.g., the correlation coefficients between neighboring variables).

\item We may extend our model in (\ref{eqn:PDFP}) to include Bayesian image analysis at the pixel level. Let $\z$ denote the observed (i.e., recorded) image, and let the
 posterior distribution of the extended model be
 \begin{multline}
 f_{\X \cond \mathbf{Z}}(\x \cond \z) \propto \\ \quad \prod_{v \in \V}\textrm{exp}\Big(-\frac{(x_v - z_v)^2}{2s^2}\Big)\!\prod_{(u,v) \in \E}\!\!\!\textrm{exp}\Big(-\frac{(x_u - x_v)^2}{2\sigma^2}\Big)
 \end{multline}
 
Markov chain Monte Carlo methods can then be used to draw samples according to the distribution. See \cite{besag1986statistical}, \cite{weiss2000correctness}, \cite[Chapter 2]{winkler2012} for
 more details.
The posterior precision matrix $\Q_{\X\cond\mathbf{Z}}=\Q+s^{-2}\I$ retains the
same structural form as~(\ref{eqn:QP2}). Therefore, the primal--dual framework of this
paper is directly applicable for posterior inference. 
In this setting, the MAP
estimate $\hat{\x} = s^{-2}\Q_{\X\cond\mathbf{Z}}^{-1}\z$ can be obtained efficiently by
exploiting the sparsity of $\Q_{\X\cond\mathbf{Z}}$ (e.g.~via the conjugate-gradient
method). The posterior uncertainty, however, requires the covariance structure or marginal variances.
The Gibbs samplers of this paper offer a complementary Monte Carlo route to estimate
these quantities in both primal and dual domains.
\end{itemize}

\section{Conclusion}

We presented a theoretical framework for analyzing the random-sweep Gibbs sampling algorithm in Gaussian graphical models with a thin-membrane prior by combining 
spectral analysis with dual normal factor graph representations.
In the primal domain, we derived the exact rates of convergence for homogeneous
$k$-regular and balanced complete bipartite graphs. In the
dual domain, we proved that, independent of the graph topology, the Gibbs sampler exhibits a universal convergence rate for all homogeneous models whose
associated graphical models contain cycles. We further demonstrated that the effective convergence rate is governed by the algebraic connectivity of the graph, yielding an additional acceleration 
in the convergence rate. Beyond the convergence analysis, we established an algebraic relation between the covariance structures of the primal and 
dual models. The algebraic relation allows us to recover the covariance structure (and therefore the marginal statistics) using samples generated directly in the dual domain.

The numerical experiments corroborate the theoretical analysis across a range of graph families and model parameters, demonstrating that the dual formulation provides substantially 
faster convergence while preserving the same $\mathcal{O}(|\E|)$ computational complexity per sweep. Our results identify dual normal factor graphs as an effective and promising
framework for accelerating the Gibbs sampling algorithm in Gaussian graphical models.

\appendices


\section{Convergence Rates for the Star Graph}
\label{appsec:StarG}

Since the star graph $S_{|\V|}$ has a tree structure, it is possible to draw independent samples according to its PDF efficiently using the sum-product algorithm. 
Nonetheless, we will derive the exact rates of convergence of the random-sweep Gibbs sampler for a homogeneous $S_{|\V|}$ as an illustrative example.

In homogeneous star graphs (as in Example 2), $\Q$ is a symmetric arrowhead matrix specified by its first row in (\ref{eqn:Star1}) and its 
diagonal in (\ref{eqn:Star2}). The entries of $\A$ in (\ref{eqn:AGen}) are given by
\begin{equation}
A_{ij}= 
\begin{cases}
\dfrac{s^2}{\sigma^2+s^2} & \text{if $i = 1$ and $1 < j \le |\V|$}\\[2mm]
\dfrac{s^2}{\sigma^2+|\E|s^2} & \text{if $j = 1$ and $1 < i \le |\V|$}\\[3mm]
0 & \text{otherwise.}
\end{cases}
\end{equation}

Let $\mathbf{v} \in \R^{|\V|}$ be an eigenvector of $\A$ and $\lambda$ be the corresponding eigenvalue. From $\A\mathbf{v} = \lambda\mathbf{v}$ we get
\begin{equation}
\frac{s^2}{\sigma^2+s^2}(v_2+v_3+\ldots+v_{|\V|}) = \lambda v_1 
\end{equation}
and for $1 < i \le |\V|$
\begin{equation}
\frac{s^2}{\sigma^2+ |\E|s^2}v_1 = \lambda v_i 
\end{equation}

Hence $v_2 = v_3 = \ldots = v_{|\V|}$. A routine calculation shows that
$\lambda(\A) \in \{0, \pm\lambda_{\text{max}}(\A)\}$, where
\begin{equation}
\lambda_{\text{max}}(\A) = \Big(\frac{|\E|}{(\sigma^2+s^2)(\sigma^2+|\E|s^2)}\Big)^{1/2}s^2
\end{equation}
Recall that $|\E| = |\V| - 1$. 

Substituting $\lambda_{\text{max}}(\A)$ in (\ref{eqn:primalrate}) gives the following asymptotic convergence rate
\begin{align}
\lim_{|\V| \to \infty} r_\text{p} & = \lim_{|\V| \to \infty} \Big(1-\frac{1}{|\V|} + \frac{s}{\sqrt{\sigma^2+s^2}}\frac{1}{|\V|}\Big)^{|\V|} \\[1mm]
                                & = \textrm{exp}\Big(\frac{s}{\sqrt{\sigma^2+s^2}} -1\Big) \label{eqn:starrateAsy}
\end{align}

In the dual of homogeneous star graphs (as in Example 4), the precision matrix $\mathbf{R}$ is 
\begin{equation}
R_{ij}= 
\begin{cases}
\sigma^2+2s^2 & \text{if $i = j$}\\[1mm]
s^2 & \text{otherwise.}
\end{cases}
\end{equation}
Thus, $\A$ in (\ref{eqn:AD}) can be expressed as
\begin{equation}
\A = \frac{s^2}{\sigma^2+2s^2}\I - \frac{s^2}{\sigma^2+2s^2}\mathbf{J}
\end{equation}
where $\mathbf{J}$ is the all-ones matrix. Since $\lambda(\mathbf{J}) \in \{0, |\E|\}$, we conclude that
\begin{equation}
\lambda_{\text{max}}(\A) = \frac{s^2}{\sigma^2+2s^2}
\end{equation}
From (\ref{eqn:dualrate}) we obtain
\begin{equation}
\label{eqn:dualrate5}
r_\text{d} = \Big(1-\frac{\sigma^2 + s^2}{\sigma^2+2s^2}\frac{1}{|\E|}\Big)^{|\E|}
\end{equation}
which gives the following asymptotic convergence rate
\begin{equation}
\label{eqn:dualrate6}
\lim_{|\E| \to \infty} r_\text{d} = \textrm{exp}\Big(-\frac{\sigma^2 + s^2}{\sigma^2+2s^2}\Big)
\end{equation}
In fact, (\ref{eqn:dualrate6}) gives a better asymptotic convergence rate than (\ref{eqn:starrateAsy}) for all finite values of 
$\sigma^2$ and $s^2$. 

The rate in~(\ref{eqn:dualrate5}) coincides with the rate derived in~\eqref{eqn:effectiverate1} for $\lambda_2(\Lap) = 1$, and 
the rate in~(\ref{eqn:dualrate6}) converges to $\text{exp}(-1/2)$ as $s \to \infty$, matching the rate reported in Table~\ref{tab:fiedler}.

Indeed, in cycle-free graphs, such as $S_{|\V|}$, the cycle space is $\mathcal N(\B)=\{\mathbf 0\}$. Therefore, $\lambda_{\text{min}}(\B^\intercal\B) > 0$ and $\B^\intercal\B$ is 
nonsingular; see Proposition~\ref{prop:Prop4}.

\section{Bounds on the Entropy Function}
\label{appsec:EntropyBound}

From Hadamard's inequality~\cite{horn2013matrix}
\begin{equation}
\label{eq:ExactdetCovUpperBounds}
\prod_{i =1}^ {|\V|}\mathbf{Q}^{-1}_{ii} \le \text{det}(\text{Cov}(\X)) \le \prod_{i =1}^ {|\V|} s_i^2 
\end{equation}

In $k$-regular homogeneous models $\mathbf{Q}_{ii} =1/s^2 + k/\sigma^2$. Hence
\begin{equation}
\Big(\frac{s^2\sigma^2}{\sigma^2 + ks^2}\Big)^{|\V|} \le \text{det}(\text{Cov}(\X)) \le s^{2|\V|}
\end{equation}

The entropy of a multivariate Gaussian distribution is given by $H(\X)  = \frac{|\V|}{2}\ln 2\pi e + \frac{1}{2}\ln\text{det}(\text{Cov}(\X))$. 
We conclude that they entropy per site is bounded by
\begin{equation}
\label{eq:ExactdetCovUpperBoundsF}
\frac{1}{2}\ln\big(2\pi e\frac{s^2\sigma^2}{\sigma^2 + ks^2}\big) \le \frac{H(\X)}{|\V|}  \le \frac{1}{2}\ln(2\pi es^2)
\end{equation}

From (\ref{eqn:YfromX}), we have $H(\Y \cond \X)  = 0$. Therefore, $H(\X) = H(\X, \Y)$ and the bounds in (\ref{eq:ExactdetCovUpperBoundsF}) hold 
for $H(\X, \Y)$ as well.

Similarly 
\begin{equation}
\big(\sigma^2 + 2s^2\big)^{-|\E|} \le \text{det}\big(\text{Cov}(\tilde{\Y})\big) \le \sigma^{-2|\E|}
\end{equation}
which gives the following bounds on the entropy per edge
\begin{equation}
\label{eq:detinequality}
\frac{1}{2}\ln\Big(\frac{2\pi e}{\sigma^2 + 2s^2}\Big) \le \frac{H(\tilde{\Y})}{|\E|}  \le \frac{1}{2}\ln \Big(\frac{2\pi e}{\sigma^2}\Big)
\end{equation}
In the dual domain, $H(\tilde{\Y}) = H(\tilde{\X},\tilde{\Y})$ as $H(\tilde{\X} \cond \tilde{\Y}) = 0$ from (\ref{eqn:XfromY}). The bounds in (\ref{eq:detinequality}) also 
hold for $H(\tilde{\X},\tilde{\Y})$.

Finally, since there are linear dependencies among the components of $\Y$ and $\tilde{\X}$, the determinants of $\text{Cov}(\Y)$ and $\text{Cov}(\tilde{\X})$ are equal to zero.

\IEEEtriggeratref{12}
\bibliography{referencefile}

\end{document}